\documentclass[twoside,11pt]{article}

\usepackage{jmlr2e}
\usepackage{amsmath}
\usepackage{times}
\usepackage{psfrag}
\usepackage{multirow}

\newcommand{\A}{\ensuremath{\mathbf{A}}}
\newcommand{\B}{\ensuremath{\mathbf{B}}}
\newcommand{\C}{\ensuremath{\mathbf{C}}}

\newcommand{\E}{\ensuremath{\mathbf{E}}}
\newcommand{\F}{\ensuremath{\mathbf{F}}}
\newcommand{\G}{\ensuremath{\mathbf{G}}}

\newcommand{\I}{\ensuremath{\mathbf{I}}}

\newcommand{\K}{\ensuremath{\mathbf{K}}}

\newcommand{\RR}{\ensuremath{\mathbf{R}}}

\newcommand{\T}{\ensuremath{\mathbf{T}}}
\newcommand{\U}{\ensuremath{\mathbf{U}}}
\newcommand{\V}{\ensuremath{\mathbf{V}}}
\newcommand{\W}{\ensuremath{\mathbf{W}}}
\newcommand{\X}{\ensuremath{\mathbf{X}}}
\newcommand{\Y}{\ensuremath{\mathbf{Y}}}
\newcommand{\Z}{\ensuremath{\mathbf{Z}}}
\renewcommand{\aa}{\ensuremath{\mathbf{a}}}

\newcommand{\f}{\ensuremath{\mathbf{f}}}
\newcommand{\g}{\ensuremath{\mathbf{g}}}

\newcommand{\p}{\ensuremath{\mathbf{p}}}
\newcommand{\q}{\ensuremath{\mathbf{q}}}

\newcommand{\uu}{\ensuremath{\mathbf{u}}}
\newcommand{\vv}{\ensuremath{\mathbf{v}}}
\newcommand{\w}{\ensuremath{\mathbf{w}}}
\newcommand{\x}{\ensuremath{\mathbf{x}}}
\newcommand{\y}{\ensuremath{\mathbf{y}}}

\newcommand{\0}{\ensuremath{\mathbf{0}}}

\newcommand{\balpha}{\ensuremath{\boldsymbol{\alpha}}}
\newcommand{\bbeta}{\ensuremath{\boldsymbol{\beta}}}

\newcommand{\bLambda}{\ensuremath{\boldsymbol{\Lambda}}}

\newcommand{\bSigma}{\ensuremath{\boldsymbol{\Sigma}}}
\newcommand{\bTheta}{\ensuremath{\boldsymbol{\Theta}}}

\newcommand{\bbE}{\ensuremath{\mathbb{E}}}

\newcommand{\bbR}{\ensuremath{\mathbb{R}}}

\newcommand{\calN}{\ensuremath{\mathcal{N}}}
\newcommand{\calO}{\ensuremath{\mathcal{O}}}

\newcommand{\norm}[1]{\left\lVert#1\right\rVert}

{%
\begin{list}{#1}{
\vspace{-\topsep}
\vspace{-\partopsep}
\setlength{\itemindent}{0cm}
\setlength{\rightmargin}{0cm}
\setlength{\listparindent}{0cm}
\settowidth{\labelwidth}{#1}
\setlength{\leftmargin}{\labelwidth}
\addtolength{\leftmargin}{\labelsep}
\setlength{\itemsep}{0cm}
}%
}%
{%
\end{list}
\vspace{-\topsep}
\vspace{-\partopsep}
}

{\begin{enumerate}%
}%
{\end{enumerate}}

\newcommand{\traceop}{\operatorname{tr}}
\newcommand{\trace}[1]{\ensuremath{\traceop\left(#1\right)}}

\graphicspath{{grf/}}

\newcommand{\ie}{i.e.\@}
\newcommand{\eg}{e.g.\@}
\newtheorem{thm}{Theorem}
\newtheorem{lem}{Lemma}

\usepackage{todonotes}
\usepackage{soul}
\newcommand{\Note}[2]{} 
\newcommand{\SideNote}[2]{} 
\renewcommand{\Note}[2]{\todo[color=#1,size=\small, inline=true]{#2}} 
\renewcommand{\SideNote}[2]{\todo[color=#1,size=\small]{#2}} 

\ShortHeadings{On Deep Multi-View Representation Learning}{Wang, Arora, Livescu and Bilmes}
\firstpageno{1}

\newcommand{\kl}[1]{}
\newcommand{\raman}[1]{}
\newcommand{\weiran}[1]{}
\newcommand{\jeff}[1]{}

\newcommand{\autoencoder}{autoencoder}
\newcommand{\autoencoders}{autoencoders}
\newcommand{\DCCA}{DCCA}
\newcommand{\DCCAE}{DCCAE}
\newcommand{\CorrAE}{CorrAE}
\newcommand{\DistAE}{DistAE}
\newcommand{\RKCCA}{FKCCA}
\newcommand{\NKCCA}{NKCCA}
\newcommand{\SVAE}{SplitAE} 
\newcommand{\CCA}{CCA}

\usepackage[bookmarks, colorlinks=true, citecolor=Violet,linkcolor=Mahogany,urlcolor=blue]{hyperref}
\newcommand{\ttt}{\ensuremath{\mathbf{t}}}
\newcommand{\bbb}{\ensuremath{\mathbf{b}}}

\begin{document}

\title{On Deep Multi-View Representation Learning:\\Objectives and Optimization}

\author{\name Weiran Wang \email weiranwang@ttic.edu \\
       \addr Toyota Technological Institute at Chicago \\
       Chicago, IL 60637, USA 
       \AND
       \name Raman Arora \email arora@cs.jhu.edu \\
       \addr Department of Computer Science \\
       Johns Hopkins University \\
       Baltimore, MD 21218, USA 
       \AND 
       \name Karen Livescu \email klivescu@ttic.edu \\
       \addr Toyota Technological Institute at Chicago \\
       Chicago, IL 60637, USA
       \AND
       \name Jeff Bilmes \email bilmes@ee.washington.edu \\
       \addr Department of Electrical Engineering \\
       University of Washington, Seattle \\
       Seattle, WA 98195, USA
       }

\editor{}

\maketitle

\begin{abstract} 
We consider learning representations (features) in the setting in which we have access to multiple unlabeled views of the data for 
learning while only one view is available 
for downstream tasks.  Previous work on this problem has proposed several techniques based on deep neural networks, typically involving either autoencoder-like networks with a reconstruction objective or paired feedforward networks with a batch-style correlation-based objective.  We analyze several techniques based on prior work, as well as new variants, and compare them empirically  on 
image, speech, and text tasks.
We find an advantage for correlation-based representation learning, while the best results on most tasks are obtained with our new variant, deep canonically correlated autoencoders (\DCCAE). We also explore a stochastic optimization procedure for minibatch correlation-based objectives and discuss the time/performance trade-offs for kernel-based and neural network-based implementations.
\end{abstract} 

\jeff{change latex compile to use latex hyperref capability.} \weiran{This is solved.}

\section{Introduction}
\label{sec:intro}

In many applications, we have access to multiple ``views'' of data at training time while only one view is available at test time,
or for a downstream task. The views can be multiple measurement modalities, such as simultaneously recorded audio + video~\citep{Kidron_05a,Chaudh_09a}, audio + articulation~\citep{AroraLivesc13a,Wang_15a}, images + text \linebreak[4] \citep{Hardoon_04a,SocherLi10a,Hodosh_13a,YanMikolaj15a}, or parallel text in two languages~\citep{Vinokour_03a,Haghig_08a,Chandar_14a,FaruquiDyer14a,Lu_15a}, but may also be different information extracted from the same source, such as words + context~\citep{Pennin_14a} or document text + text of inbound hyperlinks~\citep{BickelScheff04a}.  The presence of multiple information sources presents an opportunity to learn better representations (features) by analyzing the views simultaneously.  Typical approaches are based on learning a feature transformation of the ``primary'' view (the one available at test time) that captures useful information from the second view using a paired two-view training set.  Under certain assumptions, theoretical results exist showing the advantages of multi-view techniques for downstream tasks~\citep{KakadeFoster07a,Foster_09a,Chaudh_09a}.  
Experimentally, prior work has shown the benefit of multi-view methods on tasks such as retrieval~\citep{Vinokour_03a,Hardoon_04a,SocherLi10a,Hodosh_13a}, clustering~\citep{Blasch08a,Chaudh_09a}, and classification/recognition~\citep{Dhillon_11b,AroraLivesc13a,Ngiam_11b}. 

Recent work has introduced several approaches for multi-view representation learning based on deep neural networks (DNNs), using two main training criteria (objectives).
One type of objective is based on deep \autoencoders~\citep{HintonSalakh06a}, where the objective is to learn a compact representation that best reconstructs the inputs. In the multi-view learning scenario, it is natural to use an encoder to extract the shared representation from the primary view, and use different decoders to reconstruct each view's input features from the shared representation. 
This approach has been shown to be effective for speech and vision tasks~\citep{Ngiam_11b}.

The second main DNN-based multi-view approach is based on 
deep extensions of canonical correlation analysis (CCA,~\citealp{Hotell36a}), which learns features in two views that are maximally correlated. 
The CCA objective has been studied extensively 
and has a number of useful properties and interpretations~\citep{Borga01a,BachJordan02a,BachJordan05a,Chechik_05a}, and the optimal linear projection mappings can be obtained by solving an eigenvalue system of a matrix whose dimensions equal the input dimensionalities. To overcome the limiting power of linear projections, a nonlinear extension--kernel canonical correlation analysis (KCCA)--has also been proposed~\citep{LaiFyfe00a,Akaho01a,Melzer_01a,BachJordan02a,Hardoon_04a}. 
CCA and KCCA have long been the workhorse for multi-view feature learning and dimensionality reduction~\citep{Vinokour_03a,KakadeFoster07a,SocherLi10a,Dhillon_11b}. 
Several alternative nonlinear CCA-like approaches based on neural networks have also been proposed~\citep{LaiFyfe99a,Hsieh00a}, but the full DNN extension of CCA, termed deep CCA (DCCA,~\citealp{Andrew_13a}) has been developed only recently.  Compared to kernel methods, DNNs are more scalable to large amounts of training data and also have the potential advantage that the non-linear mapping can be learned, unlike with a fixed kernel approach.

The contributions of this paper are as follows.
  We present a head-to-head comparison of several DNN-based approaches, along with linear and kernel CCA, in the unsupervised multi-view feature learning setting where the second view is not available at test time. 
We compare approaches based on prior work, as well as developing and comparing new variants.  Empirically, we find that CCA-based approaches tend to outperform unconstrained reconstruction-based approaches.  One of the new methods we propose, a DNN-based model combining CCA and \autoencoder-based terms, is the consistent winner across several tasks.  
Finally, we study a stochastic optimization approach for deep CCA, both empirically and theoretically.  To facilitate future work, we have released our implementations and a new benchmark dataset of simulated two-view data based on MNIST.\footnote{Data and implementations can be found at \texttt{http://ttic.uchicago.edu/\~{}wwang5/dccae.html}. 
}

\jeff{needs to be a few sentences here saying the key differences from and additions over the ICML-2015 paper, and what justifies this as a separate JMLR journal paper.} \kl{done}

An early version of this work appeared in~\citep{Wang_15b}. This paper expands on that work with expanded discussion of related work (Sections~\ref{sec:related_kcca} and~\ref{sec:related_other}), additional experiments, and theoretical analysis.  In particular we include additional experiments on word similarity tasks with multilingual word embedding learning (Section~\ref{sec:NLP}); extensive empirical comparisons between batch and stochastic optimization methods for DCCA; and comparisons between \DCCA\ and two popular low-rank approximate KCCA methods, demonstrating their computational trade-offs (Section~\ref{sec:expt_opt}).  We also analyze the stochastic optimization approach for \DCCA\ theoretically in Appendix~\ref{sec:append-sgd}.

\section{DNN-based multi-view feature learning}
\label{sec:algorithm}

In this section, we discuss several existing and new multi-view learning approaches based on deep feed-forward neural networks, including their objective functions and optimization procedures. Schematic diagrams summarizing the methods are given in Fig.~\ref{f:models}.

\paragraph{Notation} In the multi-view feature learning scenario, we have access to paired observations from two views, denoted ${(\x_1,\y_1),\dots,(\x_N,\y_N)}$, where $N$ is the sample size and $\x_i\in \bbR^{D_x}$ and $\y_i\in \bbR^{D_y}$ for $i=1,\dots,N$. We also denote the data matrices for each view $\X=[\x_1,\dots,\x_N]$ and $\Y=[\y_1,\dots,\y_N]$. We use bold-face letters, \eg~$\f$, to denote multidimensional mappings implemented by DNNs.
A DNN $\f$ of depth $K_\f$ implements a nested mapping of the form $\f(\x)=\f_{K_\f}( (\cdots \f_1(\x;\W_1) \cdots);\W_{K_f})$, 
where $\W_j$ are the weight parameters\footnote{Biases can be introduced by appending an extra $1$ to the input.} of layer $j$, $j=1,\dots,K_\f$, and $\f_j$ is the mapping of layer $j$ which takes the form of a linear mapping followed by a element-wise activation: $\f_j(\ttt)=s(\W_j^\top \ttt)$, and typical choices for $s$ include sigmoid, tanh, ReLU, etc. 
The collection of the learnable parameters in model $\f$ is denoted $\W_\f$, \eg~$\W_\f=\{\W_1,\dots,\W_{K_f}\}$
in the DNN case.
We write the $\f$-projected (view 1) data matrix as $\f(\X)=[\f(\x_1),\dots,\f(\x_N)]$. The dimensionality of the projection
(feature vector), which is equal to the number of output units in the DNN case, is denoted $L$.

\begin{figure*}[t]
\centering
\psfrag{x}[][]{$\x$}
\psfrag{y}[][]{$\y$}
\begin{tabular}{@{}c@{\hspace{0\linewidth}}c@{\hspace{0.02\linewidth}}c@{}}
\psfrag{reconstructed x}[][]{Reconstructed $\x$}
\psfrag{reconstructed y}[][]{Reconstructed $\y$}
\psfrag{f}[][]{$\f$}
\psfrag{g}[][]{$\g$}
\psfrag{p}[][]{$\p$}
\psfrag{q}[r][]{$\q$}
\includegraphics[width=0.28\linewidth,height=0.3\linewidth]{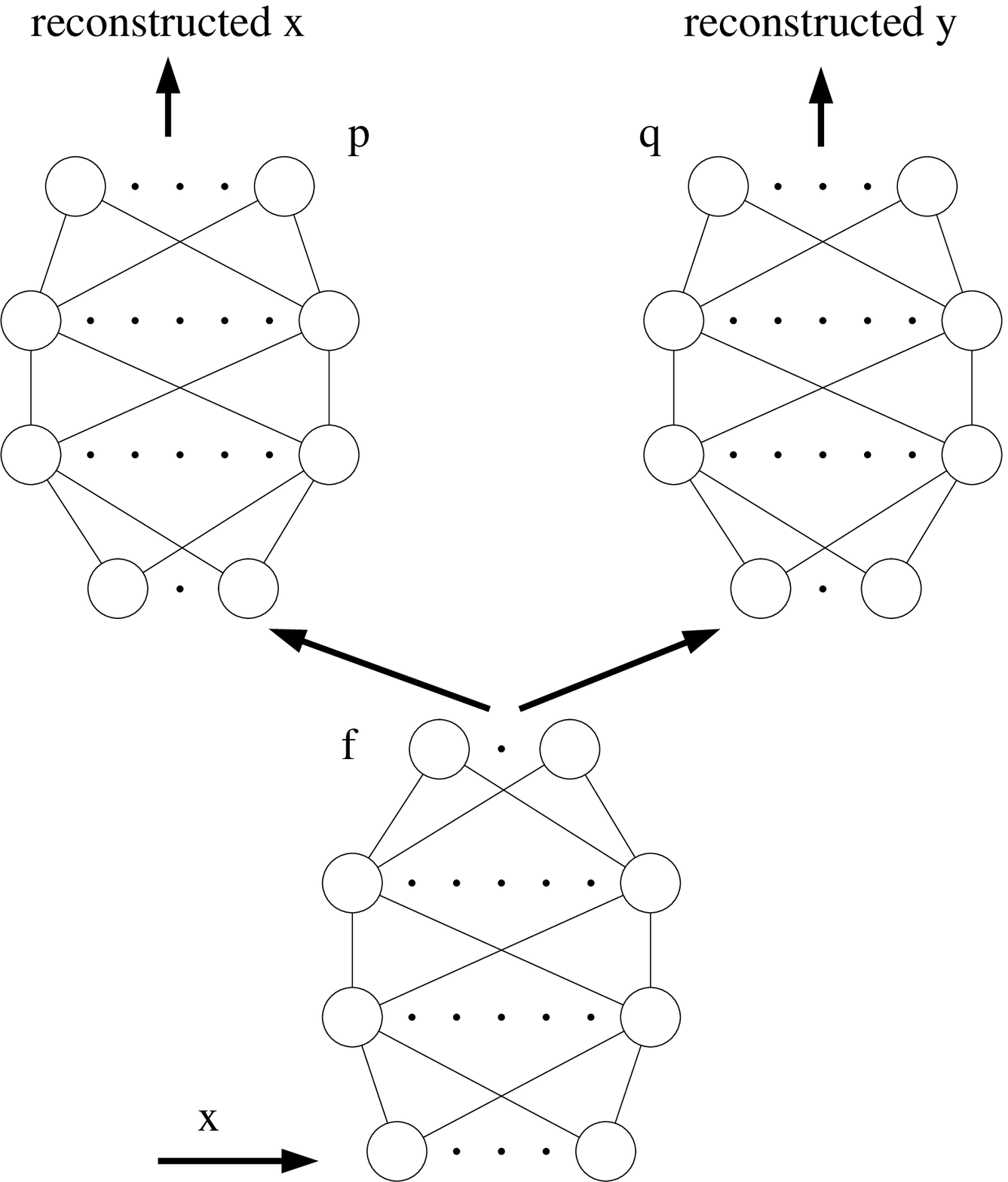} &
\psfrag{v1}[][][0.8]{View 1}
\psfrag{v2}[][][0.8][90]{View 2}
\psfrag{U}[][]{$\U$}
\psfrag{V}[][]{$\V$}
\psfrag{f}[][]{$\f$}
\psfrag{g}[][]{$\g$}
\includegraphics[width=0.35\linewidth,height=0.27\linewidth]{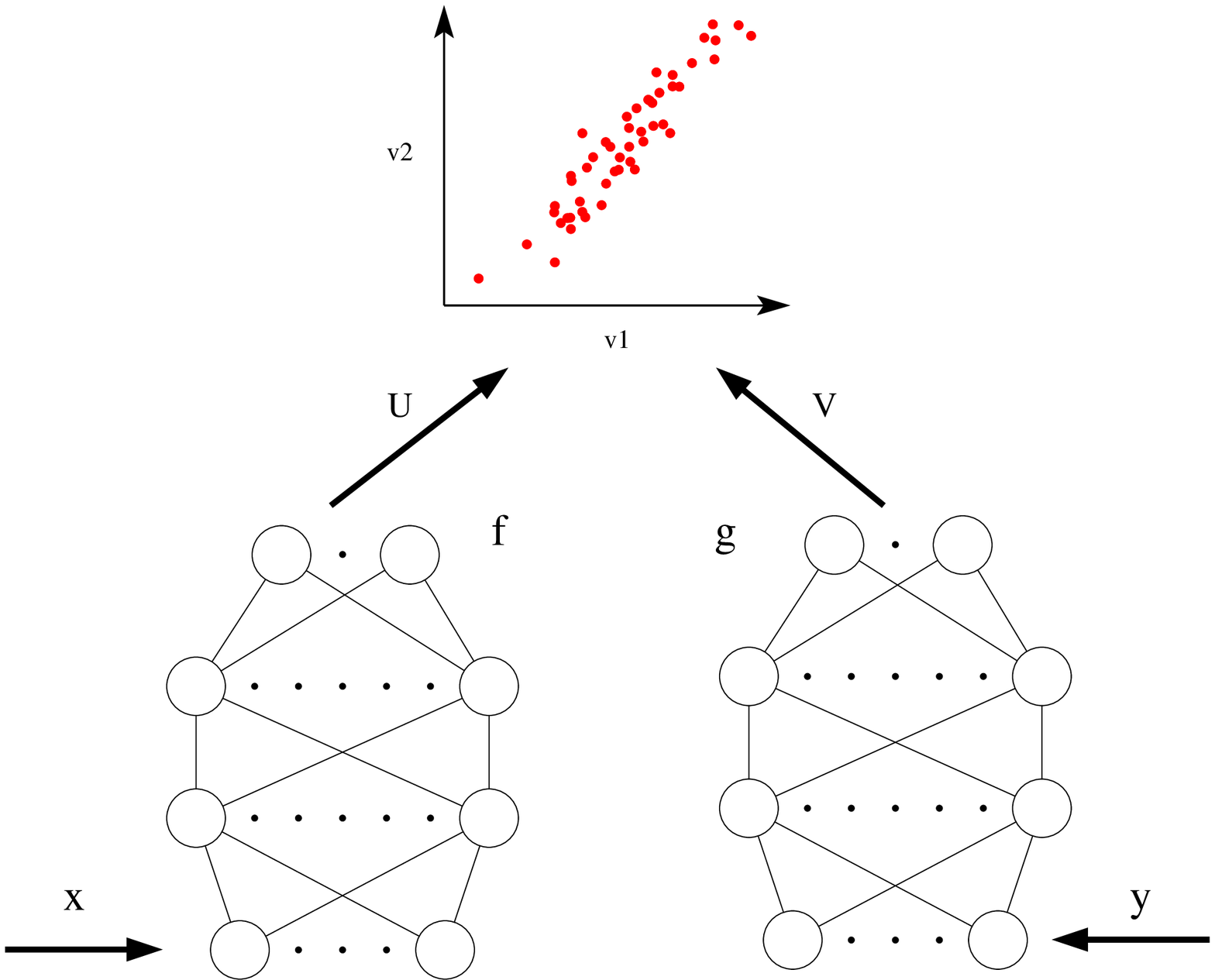} &
\psfrag{reconstructed x}[b][]{Reconstructed $\x$}
\psfrag{reconstructed y}[b][]{Reconstructed $\y$}
\psfrag{v1}[t][][0.8]{View 1}
\psfrag{v2}[][][0.8][90]{View 2}
\psfrag{U}[][][0.8]{$\U$}
\psfrag{V}[][][0.8]{$\V$}
\psfrag{f}[][]{$\f$}
\psfrag{g}[][]{$\g$}
\psfrag{p}[][]{$\p$}
\psfrag{q}[r][]{$\q$}
\includegraphics[width=0.33\linewidth,height=0.3\linewidth]{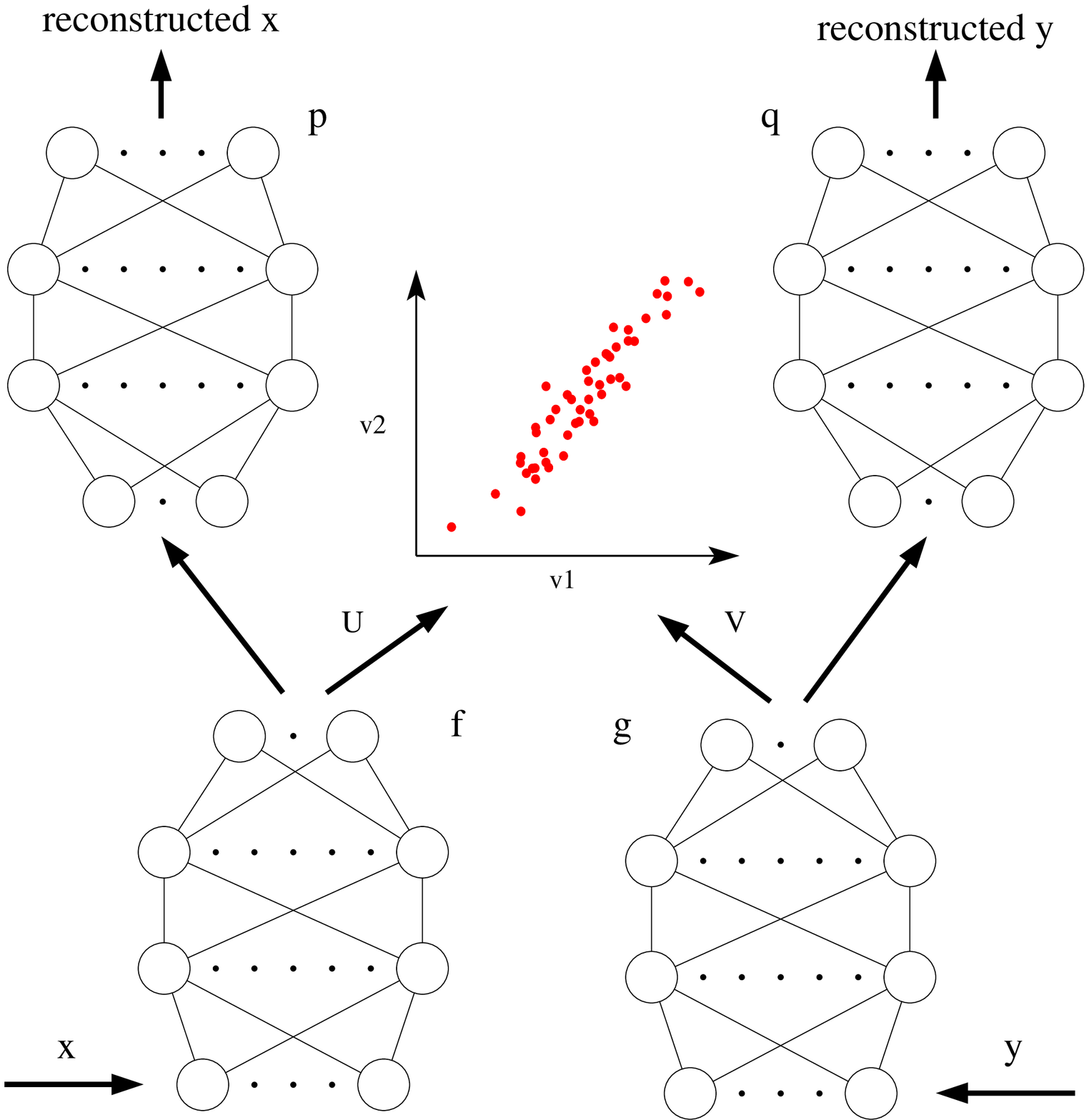} \\
(a) \SVAE & (b) \DCCA & (c) \DCCAE/\CorrAE/\DistAE
\end{tabular}
\caption{Schematic diagram of DNN-based multi-view representation learning approaches.}
\label{f:models}
\end{figure*}

\subsection{Split \autoencoders\ (\SVAE)}

\citet{Ngiam_11b} propose to extract shared representations by 
learning to reconstruct both views from the one view that is available at test time.  In this approach, the feature extraction network $\f$ is shared while the reconstruction networks $\p$ and $\q$ are separate for each view.  We refer to this model as a 
split \autoencoder\ (\SVAE), shown schematically in Fig.~\ref{f:models} (a). The objective of this model is the sum of reconstruction errors for the two views (we omit the $\ell_2$ weight decay term for all models in this section):
\begin{align} 
\min\limits_{\W_\f,\W_\p,\W_\q} \quad  \frac{1}{N}  \sum_{i=1}^N  \norm{\x_i-\p(\f(\x_i))}^2 + \norm{\y_i-\q(\f(\x_i))}^2 . \nonumber 
\end{align}
The intuition for this model is that the shared representation can be extracted from a single view, and can be used to reconstruct all views.\footnote{The authors also propose a bimodal deep \autoencoder\ combining DNN transformed features from both views; this model is more natural for the multimodal fusion setting, where both views are available at test time, but can also be used in our multi-view setting (see~\citealp{Ngiam_11b}).
Empirically, however, \citet{Ngiam_11b} report that \SVAE\ tends to work better in the multi-view setting than bimodal \autoencoders.} 
The \autoencoder\ loss is the empirical expectation of the loss incurred at each training sample, and thus stochastic gradient descent (SGD) can be used to optimize the objective efficiently with the gradients estimated on small minibatches of samples.

\subsection{Deep canonical correlation analysis (\DCCA)}
\label{sec:DCCA}

\citet{Andrew_13a} propose a DNN extension of CCA termed deep CCA (\DCCA; see Fig.~\ref{f:models} (b)).  In \DCCA, two DNNs $\f$ and $\g$ are used to extract nonlinear features for each view and the canonical correlation between the extracted features $\f(\X)$ and $\g(\Y)$ is maximized:
\begin{align} 
\max\limits_{\W_\f,\W_\g,\U,\V} \quad & \frac{1}{N}\trace{\U^\top \f(\X)  \g(\Y)^\top \V}  \label{e:dcca} \\
\text{s.t.} \quad & \U^\top \left(\frac{1}{N} { \f(\X) \f(\X)^\top}  + r_x \I\right) \U = \I, \nonumber \\
   & \V^\top \left(\frac{1}{N} {\g(\Y) \g(\Y)^\top} + r_y \I\right) \V = \I, \nonumber\\
& \uu_i^\top \f(\X) \g(\Y)^\top \vv_j=0, \quad \text{for}\quad i\neq j, \nonumber
\end{align}
where $\U=[\uu_1,\dots,\uu_L]$ and $\V=[\vv_1,\dots,\vv_L]$ are the CCA directions that project the DNN outputs and $(r_x, r_y)>0$ are regularization parameters added to the diagonal of the sample auto-covariance matrices 
\citep{BieMoor03a,Hardoon_04a}.
In \DCCA, $\U^\top \f(\cdot)$ is the final projection mapping used for 
downstream tasks.%
\footnote{In principle there is no need for the final linear projection; we could define DCCA such that the correlation objective and constraints are imposed on the final nonlinear layer of the two DNNs.  The final linear projection is equivalent to constraining the final DNN layer to be linear.  While this is in principle not needed, it is crucial for practical algorithmic implementations such as ours, and matches the original formulation of DCCA~\citep{Andrew_13a}.}
One intuition for CCA-based objectives is that, while it may be difficult to accurately reconstruct one view from the other view, it may be easier, and perhaps sufficient, to learn a predictor of a \textit{function} (or \textit{subspace}) of the second view. 
In addition, it should be helpful for the learned dimensions within each view to be uncorrelated so that they provide complementary information. 

\paragraph{Optimization} The \DCCA\ objective couples all training samples through the whitening constraints (i.e., it is a fully batch objective), so stochastic gradient descent (SGD) cannot be applied in a standard way. However, it is observed that \DCCA\  can still be optimized 
effectively as long as the gradient is estimated using a sufficiently large minibatch (\citealp{Wang_15a}, with the gradient formulas given as in \citealp{Andrew_13a}). Intuitively, this approach works because a large minibatch of samples contains sufficient information for estimating the covariance matrices. We provide empirical analysis about the optimization of \DCCA\ (in comparison with KCCA) in Section~\ref{sec:expt_opt} and its theoretical analysis is given in Appendix~\ref{sec:append-sgd}.

\subsection{Deep canonically correlated \autoencoders\ (\DCCAE)}
\label{sec:DCCAE}

Inspired by both CCA and reconstruction-based objectives, we propose a new model that consists of two \autoencoders\ and optimizes the combination of canonical correlation between the learned 
``bottleneck'' representations and the reconstruction errors of the \autoencoders. In other words, we optimize the following objective  
\begin{align} 
\min\limits_{\W_\f,\W_g,\W_\p,\W_\q,\U,\V} \quad& -\frac{1}{N}\trace{\U^\top \f(\X)  \g(\Y)^\top \V} \nonumber \\
&  + \frac{\lambda}{N}  \sum_{i=1}^N \norm{\x_i-\p(\f(\x_i))}^2 + \norm{\y_i-\q(\g(\y_i))}^2\label{e:dccae} \\
\text{s.t.}\quad  &   \text{the same constraints as in } \eqref{e:dcca}, \nonumber
\end{align}
where $\lambda>0$ is a trade-off parameter.
Alternatively, this approach can be seen as adding an \autoencoder\ regularization term to \DCCA. We call this approach deep canonically correlated \autoencoders\ (\DCCAE).  
Fig.~\ref{f:models} (c) shows a schematic representation of the approach.
\paragraph{Optimization} We apply stochastic optimization to the \DCCAE\ objective. 
Notice obtaining good stochastic estimates of the gradient for the correlation and autoencoder terms may involve different minibatch sizes.  We explore the minibatch sizes for each term separately (by training DCCA and an autoencoder) and select them using a validation set. 
The stochastic gradient is then the sum of the gradient for the \DCCA\ term (usually estimated using large minibatches) and the gradient for the \autoencoder\ term (usually estimated using small minibatches).

\paragraph{Interpretations}

CCA maximizes the mutual information between the projected views for certain distributions \citep{Borga01a}, while training an \autoencoder\ to minimize reconstruction error amounts to maximizing a lower bound on the mutual information between inputs and learned features \citep{Vincen_10a}. The \DCCAE\ objective offers a trade-off between the information captured in the (input, feature) mapping within each view on the one hand, and the information in the (feature, feature) relationship across views.  

\subsection{Correlated \autoencoders\ (\CorrAE)}

In the next approach, we remove the uncorrelatedness constraints from the \DCCAE\ objective, leaving only the sum of scalar correlations between pairs of learned dimensions and the reconstruction error term.  This approach is intended to test how important the original CCA constraints are. 
We call this model correlated \autoencoders\ (\CorrAE), 
also represented by Fig.~\ref{f:models} (c).
Its objective can be written as 
\begin{align} 
\min\limits_{\W_\f,\W_g,\W_\p,\W_\q,\U,\V} \quad& -\frac{1}{N}\trace{\U^\top \f(\X)  \g(\Y)^\top \V} \nonumber \\
& + \frac{\lambda}{N}  \sum_{i=1}^N \norm{\x_i-\p(\f(\x_i))}^2 + \norm{\y_i-\q(\g(\y_i))}^2 \label{e:corrae} \\
\text{s.t.} \quad & \uu_i^\top \f(\X) \f(\X)^\top \uu_i=\vv_i^\top \g(\Y) \g(\Y)^\top \vv_i=N, \quad 1\le i \le L. \nonumber
\end{align}
where $\lambda>0$ is a trade-off parameter.  It is clear that the constraint set in \eqref{e:corrae} is a relaxed version of that of \eqref{e:dccae}.
Later we demonstrate that this difference results in a large performance gap. We apply the same optimization strategy of \DCCAE\ to \CorrAE. 

\CorrAE\ is similar to 
the model of 
\citet{Chandar_14a,Chandar_15a}, who try to learn vectorial word representations using parallel corpora from two languages.  They use a DNN in each view (language) to predict a bag-of-words representation of the input sentences, or that of the paired sentences from the other view, while encouraging the learned bottleneck layer representations to be highly correlated.

\subsection{Minimum-distance \autoencoders\ (\DistAE)} 
\label{sec:DistAE}

The CCA objective can be seen as minimizing the distance between the learned projections of the two views, while satisfying the whitening constraints for the projections~\citep{Hardoon_04a}. The constraints complicate the optimization of CCA-based objectives, as pointed out above.  This observation motivates us to consider additional objectives that decompose into sums over training examples, while maintaining the intuition of the CCA objective as a reconstruction error between two mappings.  Here we consider two variants that we refer to as minimum-distance \autoencoders\ (\DistAE).

The first variant \DistAE-1 optimizes the following objective:
\begin{align} 
\min\limits_{\W_\f,\W_\g,\W_\p,\W_\q} \quad & \frac{1}{N} \sum_{i=1}^N \frac{\norm{\f(\x_i)-\g(\y_i)}^2}{\norm{\f(\x_i)}^2+\norm{\g(\y_i)}^2} \nonumber\\
& + \frac{\lambda}{N}  \sum_{i=1}^N  \norm{\x_i-\p(\f(\x_i))}^2 + \norm{\y_i-\q(\g(\y_i))}^2
\end{align}
which is a weighted combination of reconstruction errors of two \autoencoders\ and the average discrepancy between the projected sample pairs. The denominator of the discrepancy term is used to keep the optimization from improving the objective by simply scaling down the projections (although they can never become identically zero due to the reconstruction terms). 
This objective is unconstrained and is the 
empirical average of the loss incurred at each training sample,
so normal SGD applies using small (or any size) minibatches.

The second variant \DistAE-2 optimizes a somewhat different objective: 
\begin{align} 
\min\limits_{\W_\f,\W_\g,\W_\p,\W_\q,\A,\bbb} \quad & \frac{1}{N} \sum_{i=1}^N {\norm{\f(\x_i)-\A\g(\y_i)-\bbb}^2} \nonumber\\
& + \frac{\lambda}{N}  \sum_{i=1}^N  \norm{\x_i-\p(\f(\x_i))}^2 + \norm{\y_i-\q(\g(\y_i))}^2
\end{align}
where $\A\in\bbR^{L\times L}$ and $\bbb\in \bbR^{L}$. The underlying intuition is that the representation of the primary view can be linearly predicted from the representation of the other view. This relationship is motivated by the fact that when
$\g(\y)$ and $\f(\x)$ are perfectly linearly correlated, then there exists 
an affine transformation that can map from one to the other.
This approach, hence, alleviates the burden on $\g(\y)$ being simultaneously predictive of the output and close to $\f(\x)$ by itself.

\section{Related work}
\label{sec:related}

Here we focus on related work on multi-view feature learning using neural networks and the kernel extension of CCA. 

\subsection{Neural network feature extraction using CCA-like objectives}
\label{sec:related_nncca}

There have been several approaches to multi-view representation learning using neural networks with an objective similar to that of CCA. Under the assumption that the two views share a common cause (\eg, depth is the common cause for adjacent patches of images), \citet{BeckerHinton92a} propose to maximize a sample-based estimate of mutual information between outputs of neural networks for the two views. 
In this work, the 1-dimensional output of each network can be considered to be ``internal teaching signals'' for the other. It is less straightforward to extend their sample-based estimator of mutual information to higher dimensions, while the CCA objective is always closely related to maximal mutual information between the views (under the joint multivariate Gaussian distributions of the inputs, see~\citealp{Borga01a}). 

\citet{LaiFyfe99a} propose to optimize the correlation (rather than canonical correlation) between the outputs of networks for each view, subject to scale constraints on each output dimension.
Instead of directly solving this constrained formulation, the authors apply Lagrangian relaxation and solve the resulting unconstrained objective using SGD. Note, however, that their objective is different from that of CCA, as there are no constraints that the learned dimensions within each view be uncorrelated. \cite{Hsieh00a} proposes a neural network-based model involving three modules: one module for extracting a pair of maximally correlated one-dimensional features for the two views; and a second and third module for reconstructing the original inputs of the two views from the learned features.  In this model, the feature dimensions can be learned one after another, each learned using as input the reconstruction residual from previous dimensions. This approach is intuitively similar to \CorrAE~and \DCCA, but the three modules are each trained separately, so there is no unified objective.

\citet{Kim_12b} propose an algorithm that first uses deep belief networks and the \autoencoder\ objective to extract features for two languages independently, and then applies linear CCA to the learned features (activations at the bottleneck layer of the \autoencoders) to learn the final representation. In this two-step approach, the DNN weight parameters are not updated to optimize the CCA objective.

There has also been work on multi-view feature learning using deep Boltzmann machines \citep{SrivasSalakh14a,Sohn_14a}. The models in this work stack several layers of restricted Boltzmann machines (RBM) to represent each view, with an additional top layer that provides the joint representation. These are probabilistic graphical models, for which the maximum likelihood objective is intractable and the training procedures are more complex. Although probabilistic models have some advantages (e.g., dealing with missing values and generating samples in a natural way), DNN-based models have the advantages of a tractable objective and efficient training.

\subsection{Kernel CCA}
\label{sec:related_kcca}

Formulation \eqref{e:dcca} encompasses several variants. Obviously, \eqref{e:dcca} reduces to 
linear CCA when $\f$ and $\g$ are identity mappings. One popular nonlinear approach is kernel CCA (KCCA, \citealp{LaiFyfe00a,Akaho01a,Melzer_01a,BachJordan02a,Hardoon_04a}), corresponding to choosing $\f(\x)$ and $\g(\y)$ to be feature maps induced by positive definite kernels $k_x(\cdot,\cdot)$ and $k_y(\cdot,\cdot)$, respectively (e.g., Gaussian RBF kernel $k(\aa,\bbb)=e^{-\norm{\aa-\bbb}^2/2 s^2}$ where $s$ is the kernel width). 
\weiran{I removed the statements regarding representer theorem of an RKHS~\citep{SchoelSmola01a}, it follows from a two-line derivation.}
Following the derivation of~\citet[Section~3.2.1]{BachJordan02a}, it suffices to consider projection mappings of the form
$\tilde{\f}(\x)=\sum_{i=1}^N {\balpha_i k_x(\x,\x_i)}$ and $\tilde{\g}(\y)=\sum_{i=1}^N {\bbeta_i k_y(\y,\y_i)}$, where $\balpha_i, \bbeta_i \in \bbR^L$, $i=1,\dots,N$.

Denote by $\K_x$ the $N \times N$ kernel matrix for view 1, \ie, $(\K_x)_{ij}=k_x (\x_i, \x_j)$, and similarly denote by $\K_y$ the kernel matrix for view 2. Then \eqref{e:dcca} can be written as a problem in the coefficient matrices $\A=[\balpha_1,\dots,\balpha_L]^\top\in \bbR^{N\times L}$ and $\B=[\bbeta_1,\dots,\bbeta_L]^\top\in \bbR^{N\times L}$:
\begin{align} \label{e:kcca}
 & \qquad \max_{\A,\B} \quad \frac{1}{N} \trace{\A^\top \K_x  \K_y \B} \\
\text{s.t.} \quad & \A^\top \left(\frac{1}{N} {\K_x}^2  + r_x \K_x\right) \A = \B^\top \left(\frac{1}{N} {\K_y}^2 + r_y \K_y \right) \B = \I, \nonumber\\
& \balpha_i^\top \K_x \K_y \bbeta_j=0, \quad \text{for}\quad i\neq j. \nonumber
\end{align}
Therefore we can conveniently work with the kernel (Gram) matrices instead of possibly infinite dimensional RKHS space and optimize directly over the coefficients. Following a derivation similar to that of CCA, one can show that the optimal solution $(\A^*, \B^*)$ satisfies
\begin{align}
(\K_x + N r_x\I)^{-1} \K_y (\K_y + N r_y\I)^{-1} \K_x \A^* &= \A^* \bSigma^2, \\
(\K_y + N r_y\I)^{-1} \K_x (\K_x + N r_x\I)^{-1} \K_y \B^* &= \B^* \bSigma^2,
\end{align}
where $\bSigma$ is a diagonal matrix containing the leading correlation coefficients. Thus the optimal projection can be obtained by solving an eigenvalue problem of size $N\times N$.

We can make a few observations on the KCCA method. First, the non-zero regularization parameters $(r_x, r_y)>0$ are needed to avoid trivial solutions and correlations. Second, in KCCA, 
the mappings are not optimized over except that the kernel parameters are usually cross-validated.  
\weiran{clarify the previous sentence.}%
\weiran{I do not remember what was going on with this comment. Shall we say ``the feature mappings $\f$ and $\g$ are fixed once a kernel and its parameters are chosen, and thus the mappings are not fully optimized over except that the kernel parameters are usually cross-validated.'' ?? But I guess the statement regarding function classes is also true for DNN to some extend. Once the architecture of DNN is chosen, it limits the DNN's power.} \kl{I think we can remove these comments}
Third, exact KCCA is computationally challenging for large data sets as it would require performing an eigendecomposition of an $N\times N$ matrix 
which is expensive both in memory (storing the kernel matrices) and time (solving the $N\times N$ eigenvalue systems naively costs $\calO(N^3)$). 

Various kernel approximation techniques have been proposed to scale up kernel machines. Two widely used approximation techniques are random Fourier features \citep{Lopez_14b} and the Nystr\"om approximation \citep{WilliamSeeger01a}. 
In random Fourier features, we randomly sample $M$ $D_x$ (respectively $D_y$)-dimensional vectors from a Gaussian distribution and map the original inputs to $\bbR^M$ by computing the dot products with the random samples followed by an elementwise cosine.  That is, $\f(\x)=[\cos(\w_1 \x+b_1),\dots,\cos(\w_M \x+b_M)]$ where $\w_i$ is sampled from $\calN(\0,\I/s^2)$ ($s$ is the width of the Gaussian kernel we wish to approximate), and $b_i$ is sampled uniformly from $[0,2\pi]$.  Inner products between transformed samples then approximate kernel similarities between original inputs. 
In the Nystr\"om approximation, we randomly select $M$ training samples $\tilde{\x}_1,\dots,\tilde{\x}_M$ and construct the $M\times M$ kernel matrix $\tilde{\K}_x$ based on these samples, i.e.~$(\tilde{\K}_x)_{ij}=k_x(\tilde{\x}_i,\tilde{\x}_j)$.  We compute the eigenvalue decomposition $\tilde{\K}_x=\tilde{\RR} \tilde{\bLambda} \tilde{\RR}^\top$, and then the $N\times N$ kernel matrix for the entire training set can be approximated as $\K_x \approx \C \tilde{\K}_x^{-1} \C^\top$ where $\C$ contains the columns of $\K_x$ corresponding to the selected subset, \ie, $\C_{ij}=k_x(\x_i,\tilde{\x}_j)$. This means $\K_x \approx \left(\C \tilde{\RR} \tilde{\bLambda}^{-1/2}\right)\left(\C \tilde{\RR} \tilde{\bLambda}^{-1/2}\right)^\top $, so we can use the $M\times N$ matrix $\F=\left(\C \tilde{\RR} \tilde{\bLambda}^{-1/2}\right)^\top$ as a new feature representation for view 1, where inner products between samples approximate kernel similarities. 

Both techniques produce rank-$M$ approximations of the kernel matrices with computational complexity $\calO(M^3 + M^2 N)$; but random Fourier features are data independent and more efficient to generate while Nystr\"om tends to work better. Other approximation techniques such as incomplete Cholesky decomposition \citep{BachJordan02a}, partial Gram-Schmidt \citep{Hardoon_04a}, incremental SVD \citep{AroraLivesc12a} have also been proposed and applied to KCCA.  However, for very large training sets, such as the ones in some of our tasks below, it remains difficult and costly to approximate KCCA well. Although recently iterative algorithms have been introduced for very large CCA
problems \citep{LuFoster14a}, they are aimed at sparse matrices and do not have a natural out-of-sample extension.

\subsection{Other related models}
\label{sec:related_other}
CCA is related to metric learning in a broad sense. In metric learning, the task is to learn a metric in the input space (or equivalently a projection mapping) such that the learned distances (or equivalently Euclidean distances in the projected space) between ``similar'' samples are small while distances between ``dissimilar'' samples are large. Metric learning often uses side information in the form of pairs of similar/dissimilar samples, which may be known {\it a priori} or derived from class labels \citep{Xing_02a,Shental_02a,Bar-Hil_05a,TsangKwok03a,SchultJoachim04a,Goldber_05a,Hoi_06a,GloberRoweis06a,Davis_07a,WeinberSaul09a}.

Note that we can equivalently write the CCA objective as (by replacing $\max$ with $\min-$ and adding $1/2$ times the left-hand side of the whitening constraints)
\begin{align} \label{e:cca2}
 \quad \min_{\U,\V} & \quad \frac{1}{N} \norm{\U^\top \F - \V^\top \G}_F^2 + \frac{r_x}{2}\norm{\U}_F^2 + \frac{r_y}{2}\norm{\V}_F^2 \\
\text{s.t.} & \quad \text{the same constraints in \eqref{e:dcca}}. \nonumber
\end{align}
In view of the above formulation, pairs of co-occurring two-view samples are mapped into similar locations in the projection, which are thus considered ``similar'' in CCA, and the whitening constraints set the 
scales of the projections so that the dataset does not collapse into a constant and so that the projection dimensions are uncorrelated.
Unlike the typical metric learning setting, the two-view data in CCA may come from different domains/modalities and thus each view has its own projection mapping. Also, CCA uses no information regarding ``dissimilar'' pairs of two-view data.  In this sense the CCA setting is more similar to that of \citet{Shental_02a} and \citet{Bar-Hil_05a}, which use only side information regarding groups of similar samples (``chunklets'') for single-view data.

For multi-view data, \citet{Glober_07a} propose an algorithm for learning Euclidean embeddings by defining a joint or conditional distribution of the views based on Euclidean distance in the embedding space and maximizing the data likelihood. The objective they minimize is the weighted mean of squared distances between embeddings of co-occurring pairs with a regularization term depending on the partition function of the defined distribution.
This model differs from CCA in the global constraints/regularization used. 

Recently, there has been increasing interest in learning (multi-view) representations using contrastive losses which aim to enforce that distances between dissimilar pairs are larger than distances between similar pairs by some margin~\citep{HermanBlunsom14a,Huang_13b}.  In case there is no ground-truth information regarding similar/dissimilar pairs (as in our setting), random sampling is typically used to generate negative pairs.  The sampling of negative pairs can be harmful in cases where the probability of mistakenly obtaining similar pairs from the sampling procedure is relatively high.  In future work it would be interesting to compare contrastive losses with the CCA objective when only similar pairs are given, and to consider the effects of such sampling in a variety of data distributions.

Finally, CCA has a connection with the information bottleneck method~\citep{Tishby_99a}.  Indeed, in the case of Gaussian variables, the information bottleneck method finds the same subspace as CCA~\citep{Chechik_05a}.

\section{Experiments}
\label{sec:experiments}

We first demonstrate the proposed algorithms and related work on several multi-view feature learning tasks (Sections~\ref{sec:MNIST}--\ref{sec:NLP}). In our setting, the second view is not available during test time, so we try to learn a feature transformation of the first/primary view that captures useful information from the second view using a paired two-view training set. Then, in Section~\ref{sec:expt_opt}, we explore the stochastic optimization procedure for the DCCA objective \eqref{e:dcca}.

We focus on several downstream tasks including noisy digit image
classification, speech recognition, and word pair semantic
similarity. On these tasks, we compare the following methods in the
multi-view learning setting:
\begin{itemize}
\item \textbf{DNN-based models}, including \SVAE, \CorrAE, \DCCA, \DCCAE, and \DistAE.
\item \textbf{Linear CCA (\CCA)},  corresponding to \DCCA\ with only a linear network with no hidden layers for both views.
\item \textbf{Kernel CCA approximations}. 
Exact KCCA is computationally infeasible for our tasks since they are large; we instead implement two kernel approximation techniques, using Gaussian RBF kernels. The first implementation, denoted \textbf{\RKCCA}, uses random Fourier features \citep{Lopez_14b} and the second implementation, denoted \textbf{\NKCCA}, uses the Nystr\"om approximation \citep{WilliamSeeger01a}. As described in Section~\ref{sec:related_kcca}, in both \RKCCA/\NKCCA, we transform the original inputs to an $M$-dimensional feature space where the inner products between samples approximate the kernel similarities \citep{Yang_12e}. We apply linear CCA to the transformed inputs to obtain the approximate KCCA solution.  
\end{itemize}

\subsection{Noisy MNIST digits}
\label{sec:MNIST}

\begin{figure}[t]
\centering
\begin{tabular}{cc}
\includegraphics[width=0.45\linewidth]{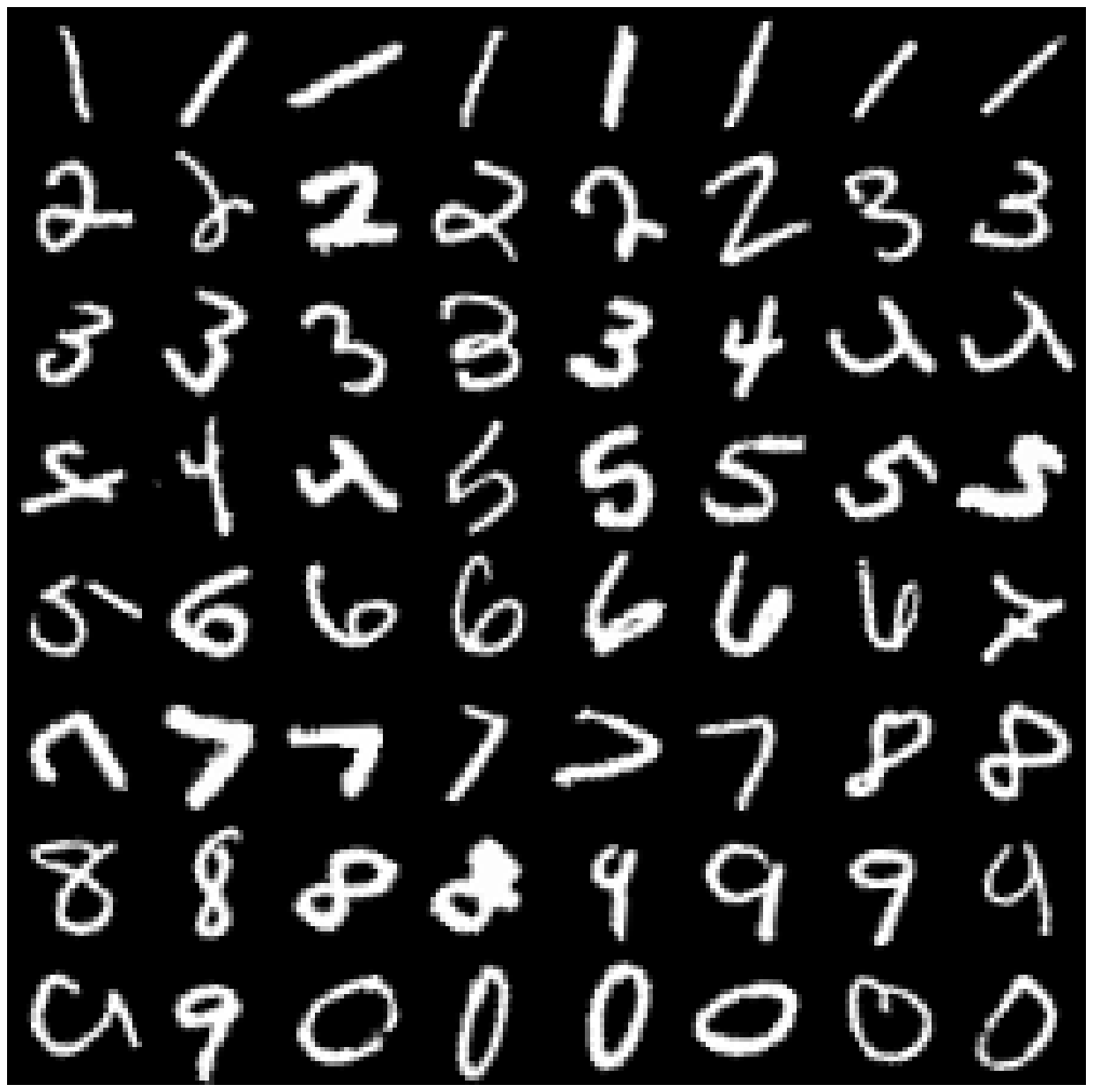} &
\includegraphics[width=0.45\linewidth]{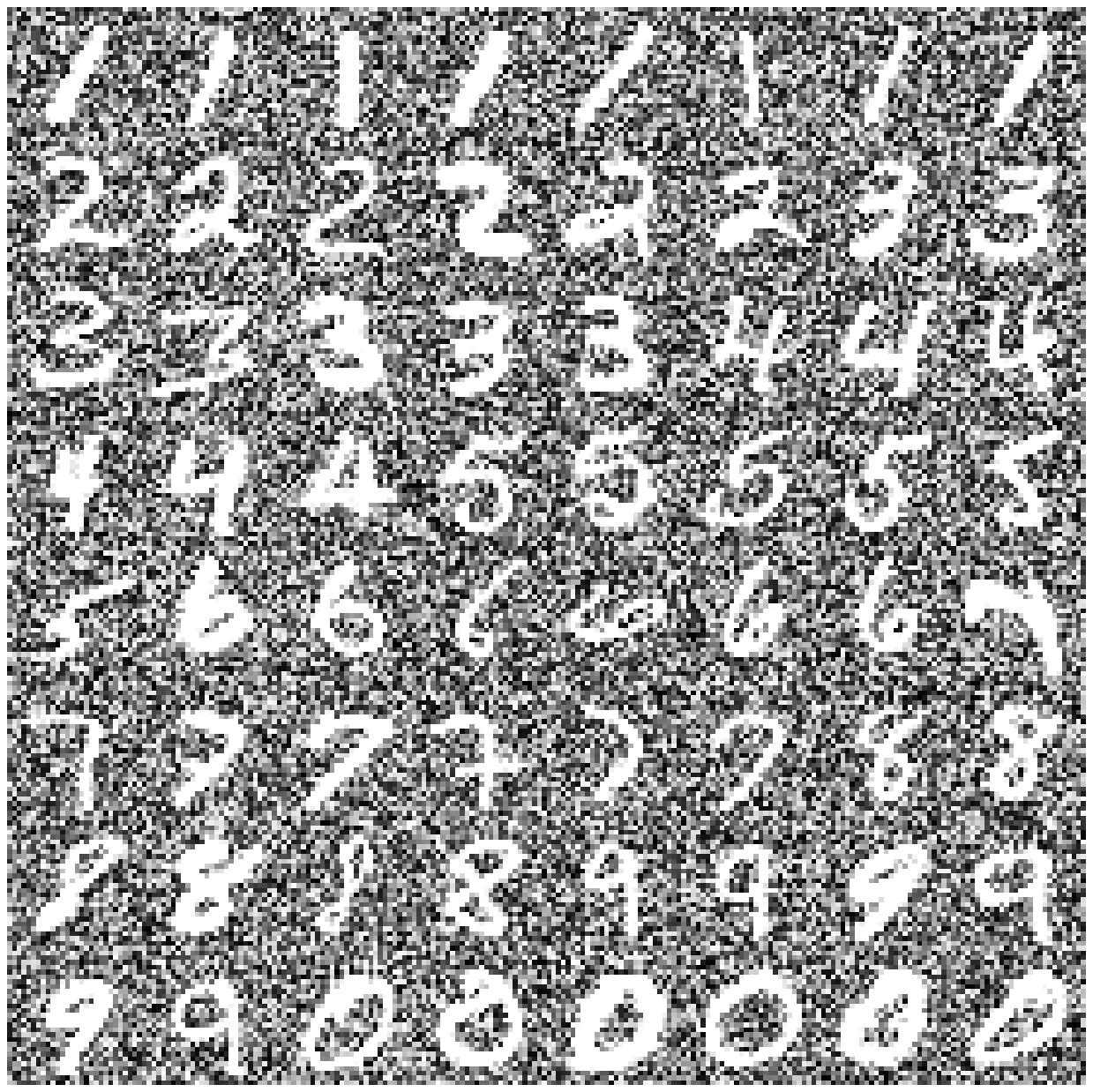}
\end{tabular}
\caption{Selection of view 1 images  (left)  and their corresponding view 2 images (right) from our noisy MNIST dataset.}
\label{f:mnist-inputs}
\end{figure}

In this task, we generate two-view data using the MNIST dataset \citep{Lecun_98a}, which consists of $28\times 28$ grayscale digit images, with $60K$/$10K$ images for training/testing. We generate a more challenging version of the dataset as follows (see Fig.~\ref{f:mnist-inputs} for examples).
We first rescale the pixel values to $[0,1]$ (by dividing the original values in $[0,255]$ by $255$).
We then randomly rotate the images at angles uniformly sampled from $[-\pi/4, \pi/4]$ and the resulting images are used as view 1 inputs. For each view 1 image, we randomly select an image 
\jeff{any chance the same base image is chosen, or is it a random sample from all images of the same identity except for the current image? Add wording saying so here.}\weiran{That is a good point. I believe I did not rule out the current image for creating view 2, but the chance of getting the same image is about 1/10000. So I think this is not a problem.}
of the same identity (0-9) from the original dataset, add independent random noise uniformly sampled from  $[0,1]$ to each pixel, and truncate the pixel final values to $[0,1]$ to obtain the corresponding view 2 sample. The original training set is further split into training/tuning sets of size $50K$/$10K$.

Since, given the digit identity, observing a view 2 image does not provide any information about the corresponding view 1 image, a good multi-view learning algorithm should be able to extract features that disregard the noise.  
We measure the class separation in the learned feature spaces via both clustering and classification performance.  First, we cluster the projected view 1 inputs into $10$ clusters and evaluate how well the clusters agree with ground-truth labels. 
We use spectral clustering \citep{Ng_02a} so as to account for possibly non-convex cluster shapes. Specifically, we first build a $k$-nearest-neighbor graph on the projected view 1 tuning/test samples with a binary weighting scheme (edges connecting neighboring samples have a constant weight of $1$), and then embed these samples in $\bbR^{10}$ using eigenvectors of the normalized graph Laplacian, and finally run $K$-means in the embedding to obtain a hard partition of the samples. In the last step,  $K$-means is run $20$ times with random initialization and the run with the best $K$-means objective is used. The size of the neighborhood graph $k$ is selected from $\{5,10,20,30,50\}$ using the tuning set.  We measure clustering performance with two criteria, clustering accuracy (AC) and normalized mutual information (NMI)~\citep{Xu_03a,Mannin_08a}. 
The AC is defined as
\begin{equation}
AC=\frac{\sum_{i=1}^N \delta(s_i,\text{map}(r_i))}{N},
\end{equation}
where $s_i$ is the ground truth label of sample $\x_i$, $r_i$ is the cluster label of $\x_i$, and 
$\text{map}(r_i)$ 
is an optimal permutation mapping between cluster labels and ground truth labels obtained by solving a linear assignment problem using the Hungarian algorithm~\citep{Munkres57a}.  The NMI considers the probability distribution over the ground truth label set $C$ and cluster label set $C^\prime$ jointly, and is defined by the following set of equations
\begin{gather}
NMI(C,C^\prime) = \frac{MI(C,C^\prime)}{\max(H(C),H(C^\prime))} \\
\text{where}\qquad\qquad 
MI(C,C^\prime)=\sum_{c_i\in C} \sum_{c_j^\prime \in C^\prime} p(c_i,c^\prime_j) \log_2 \frac{p(c_i,c^\prime_j)}{p(c_i)p(c^\prime_j)} \\
H(C)=-\sum_{c_i\in C} p(c_i) \log_2 p(c_i),\qquad
H(C^\prime)=-\sum_{c^\prime_j\in C^\prime} p(c^\prime_j) \log_2 p(c^\prime_j),
\end{gather}
where $p(c_i)$ is interpreted as the probability of a sample having label $c_i$ and $p(c_i,c^\prime_j)$ the probability of a sample having label $c_i$ while being assigned to cluster $c^\prime_j$ (all of which can be computed by counting the samples in the joint partition of $C$ and $C'$
).
Larger values of these criteria (with an upper bound of 1) indicate better agreement between the clustering and ground-truth labeling. 

Each algorithm has hyperparameters that are selected using the tuning set.  The final dimensionality $L$ is selected from $\{5,10,20,30,50\}$. For CCA, the regularization parameters $r_x$ and $r_y$ are selected via grid search. For KCCAs, we fix both $r_x$ and $r_y$ at a small positive value of $10^{-4}$ (as suggested by \citet{Lopez_14b}, \RKCCA\ is robust to $r_x$, $r_y$
), and do grid search for the Gaussian kernel width over $\{2,3,4,5,6,8,10,15,20\}$ for view 1 and $\{2.5, 5, 7.5, 10, 15, 20, 30\}$ for view 2 at rank $M=5,000$, and then test with $M=20,000$. For DNN-based models, the feature mappings $(\f, \g)$ are implemented by networks of $3$ hidden layers, each of $1024$ sigmoid units, and a linear output layer of $L$ units; reconstruction mappings $(\p, \q)$ are implemented by networks of $3$ hidden layers, each of $1024$ sigmoid units, and an output layer of $784$ sigmoid units. We fix $r_x=r_y=10^{-4}$ for \DCCA\ and \DCCAE. 
For \SVAE/\CorrAE/\DCCAE/\DistAE\ we select the trade-off parameter $\lambda$ via grid search over \{0.001, 0.01, 0.1, 1, 10\}, allowing a trade-off between the correlation term (with value in $[0,L]$) and the reconstruction term (varying roughly in the range $[60,110]$ as the reconstruction error for the second view is always large).  The networks $(\f, \p)$ are pre-trained in a layerwise manner using restricted Boltzmann machines \citep{HintonSalakh06a} and similarly for $(\g, \q)$ with inputs from the corresponding view.

\begin{table}[t]
\centering
\caption{Performance of several representation learning methods on the noisy MNIST digits test set.  Performance measures are clustering accuracy (AC), normalized mutual information (NMI) of clustering, and classification error rates of a linear SVM on the projections. The selected feature dimensionality $L$ is given in parentheses. Results are averaged over $5$ random seeds.}
\label{t:mnist_acc_nmi}
\begin{tabular}{l|c|c|c}
\hline
Method & AC (\%) & NMI (\%) & Error (\%) \\
\hline \hline
Baseline & 47.0 & 50.6 & 13.1\\
\CCA\ ($L=10$) & 72.9 & 56.0 & 19.6 \\
\SVAE\ ($L=10$) & 64.0 & 69.0 & 11.9 \\
\CorrAE\ ($L=10$) & 65.5 & 67.2 & 12.9 \\
\DistAE-1\ ($L=20$) & 53.5 & 60.2 & 16.0 \\
\DistAE-2\ ($L=10$) & 62.6 & 65.6 & 15.2 \\
\RKCCA ($L=10$) & 94.7 & 87.3 & \hspace{.4em} 5.1 \\
\NKCCA ($L=10$) & 95.1 & 88.3 & \hspace{.4em}  4.5 \\
\DCCA\ ($L=10$) & \textbf{97.0} & \textbf{92.0} & \hspace{.4em} \textbf{2.9} \\
\DCCAE\ ($L=10$) & \textbf{97.5} & \textbf{93.4} & \hspace{.4em} \textbf{2.2} \\
\hline
\end{tabular}
\end{table}

\begin{figure*}[htp]
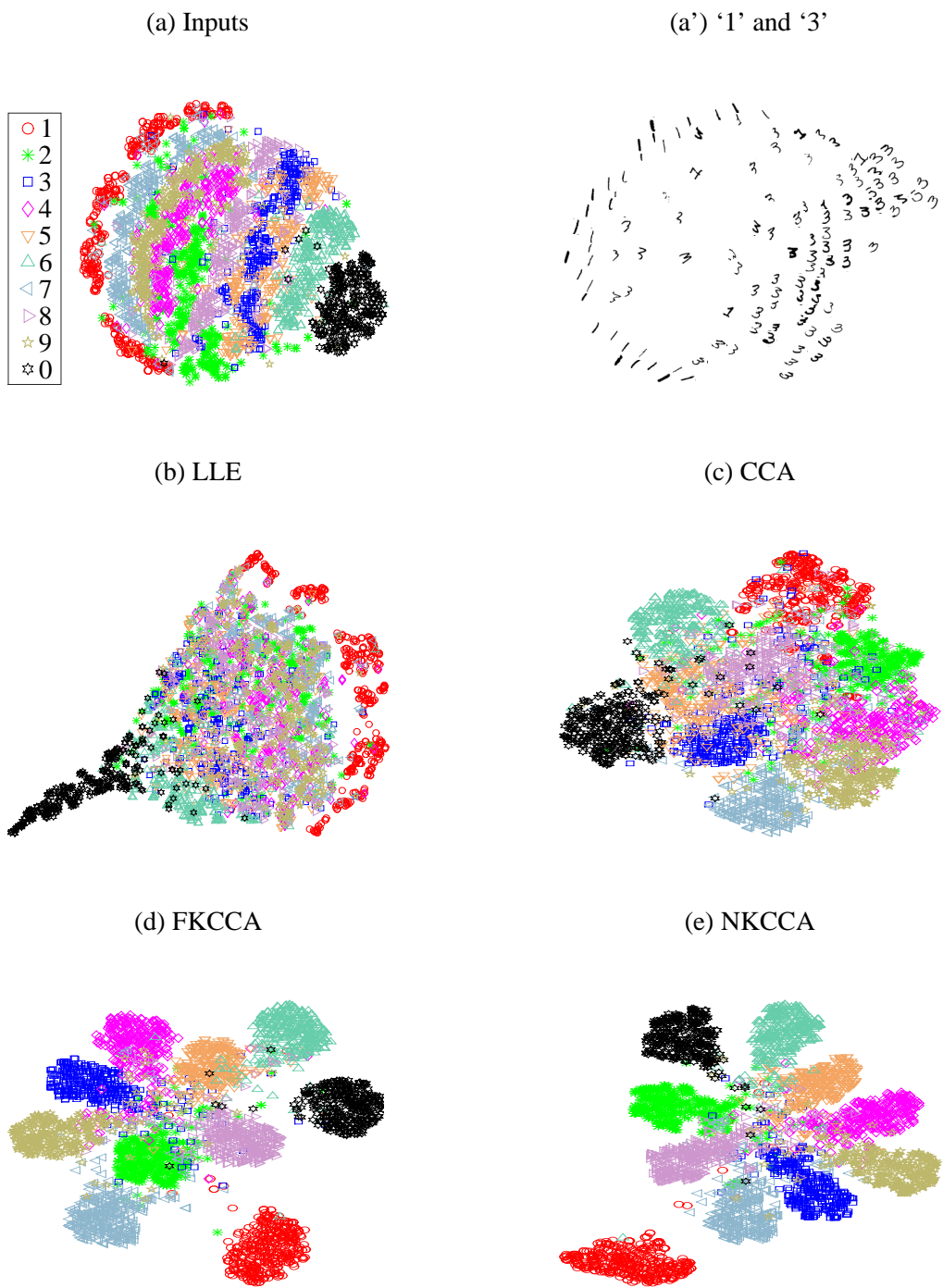

\centering
\begin{tabular}{@{}c@{\hspace{0.16\linewidth}}c@{}}
(a) Inputs & (a') `1' and `3' \\[5ex]
\psfrag{1}[][l]{1}
\psfrag{2}[][l]{2}
\psfrag{3}[][l]{3}
\psfrag{4}[][l]{4}
\psfrag{5}[][l]{5}
\psfrag{6}[][l]{6}
\psfrag{7}[][l]{7}
\psfrag{8}[][l]{8}
\psfrag{9}[][l]{9}
\psfrag{0}[][l]{0}
\includegraphics[width=0.36\linewidth,height=0.27\linewidth]{MNIST_INPUT_2D.eps} &
\includegraphics[width=0.36\linewidth,height=0.27\linewidth]{MNIST_INPUT_1A3.eps} \\[5ex]
(b) LLE & (c) CCA \\[5ex]
\includegraphics[width=0.36\linewidth,height=0.27\linewidth]{MNIST_LLE_2D.eps} & 
\includegraphics[width=0.36\linewidth,height=0.27\linewidth]{MNIST_LCCA_2D.eps} \\[5ex]
(d) \RKCCA & (e) \NKCCA \\[5ex]
\includegraphics[width=0.36\linewidth,height=0.27\linewidth]{MNIST_RKCCA_2D.eps} &
\includegraphics[width=0.36\linewidth,height=0.27\linewidth]{MNIST_NKCCA_2D.eps} 
\end{tabular}
\caption{$t$-SNE embedding of the projected test set of noisy MNIST digits by different algorithms. Each sample is denoted by a marker located at its embedding coordinates and color-coded by its identity, except in (a') where the actual input images are shown for samples of classes `1' and `3'. Neither the feature learning nor $t$-SNE is aware of the class information. We run the $t$-SNE implementation of \citet{VladymCarreir12a} on projected test images for $300$ iterations with a perplexity of $20$.}
\label{f:mnist_all_1}
\end{figure*}

\begin{figure*}[htp]
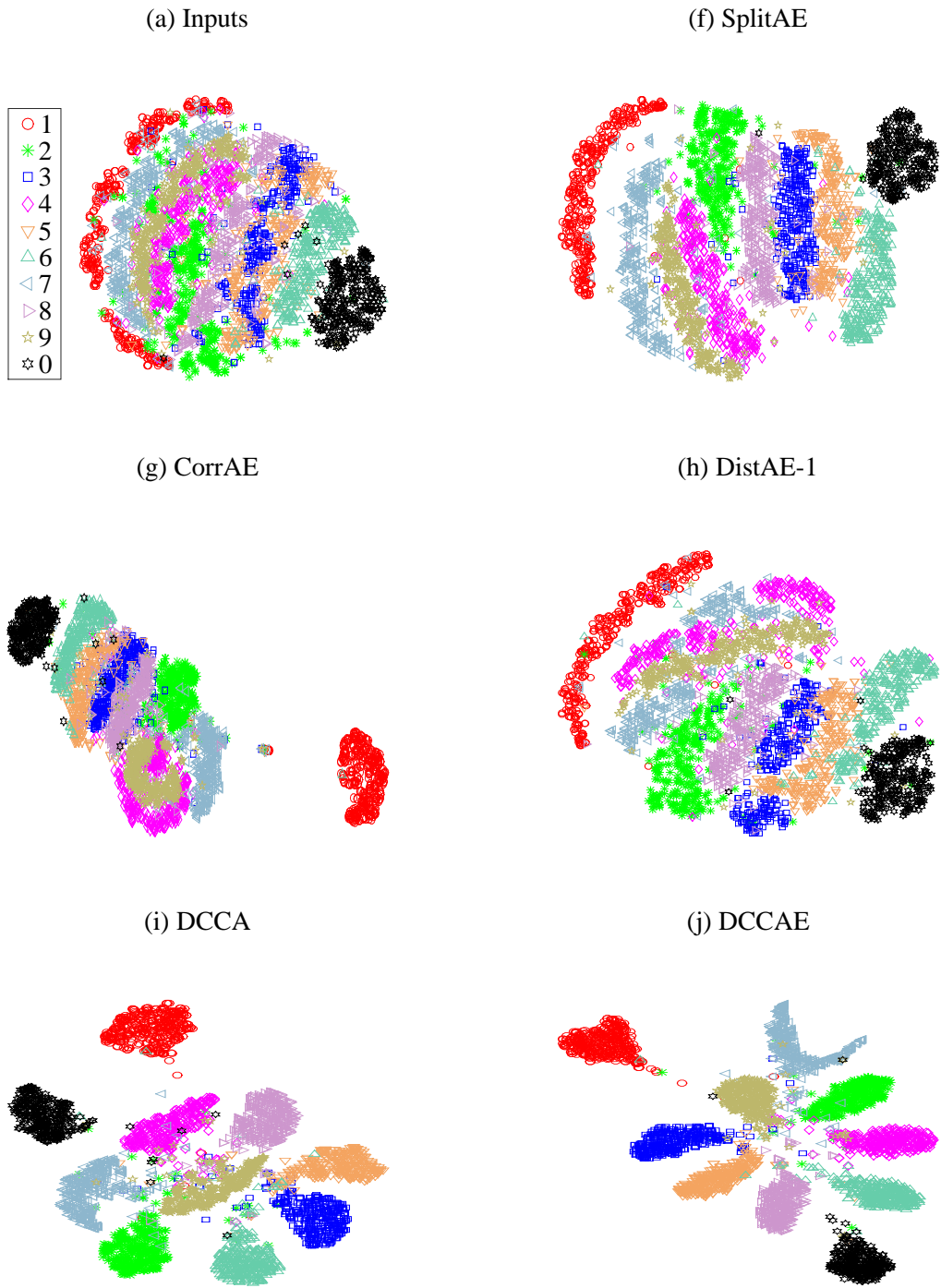

\centering
\begin{tabular}{@{}c@{\hspace{0.16\linewidth}}c@{}}
(a) Inputs & (f) \SVAE  \\[5ex]
\psfrag{1}[][l]{1}
\psfrag{2}[][l]{2}
\psfrag{3}[][l]{3}
\psfrag{4}[][l]{4}
\psfrag{5}[][l]{5}
\psfrag{6}[][l]{6}
\psfrag{7}[][l]{7}
\psfrag{8}[][l]{8}
\psfrag{9}[][l]{9}
\psfrag{0}[][l]{0}
\includegraphics[width=0.36\linewidth,height=0.27\linewidth]{MNIST_INPUT_2D.eps} &
\includegraphics[width=0.36\linewidth,height=0.27\linewidth]{MNIST_SVAUTO_2D.eps} \\[5ex]
(g) \CorrAE & (h) \DistAE-1 \\[5.5ex]
\includegraphics[width=0.36\linewidth,height=0.23\linewidth]{MNIST_CA_2D.eps} &
\includegraphics[width=0.36\linewidth,height=0.27\linewidth]{MNIST_DIS_2D.eps} \\[5ex] 
(i) \DCCA & (j) \DCCAE \\[5ex]
\includegraphics[width=0.36\linewidth,height=0.27\linewidth]{MNIST_DCCA_2D.eps} &
\includegraphics[width=0.36\linewidth,height=0.27\linewidth]{MNIST_DCCA_AE_2D.eps}
\end{tabular}
\caption{$t$-SNE embedding of the projected test set of noisy MNIST digits by different algorithms. Each sample is denoted by a marker located at its coordinates of embedding and color coded by its identity. Neither the feature learning algorithms nor $t$-SNE is aware of the class information.} 
\label{f:mnist_all_2}
\end{figure*}


\begin{figure}[t]
\psfrag{lambda}[c][c]{$\lambda$ ($\log$ scale)}
\psfrag{accuracy}[c][c]{Clustering accuracy (\%)}
\psfrag{1}[c][c][0.8]{$0$}
\psfrag{2}[c][c][0.8]{\hspace{-.6em}$.0001$}
\psfrag{3}[c][c][0.8]{\hspace{-.1em}$.0002$}
\psfrag{4}[c][c][0.8]{$.0005$}
\psfrag{5}[c][c][0.8]{$.001$}
\psfrag{6}[c][c][0.8]{$.002$}
\psfrag{7}[c][c][0.8]{$.005$}
\psfrag{8}[c][c][0.8]{$.01$}
\psfrag{9}[c][c][0.8]{$.02$}
\psfrag{10}[c][c][0.8]{$.05$}
\psfrag{11}[c][c][0.8]{$.1$}
\psfrag{12}[c][c][0.8]{$1$}
\psfrag{13}[c][c][0.8]{$10$}
\centering

\includegraphics[width=0.7\linewidth]{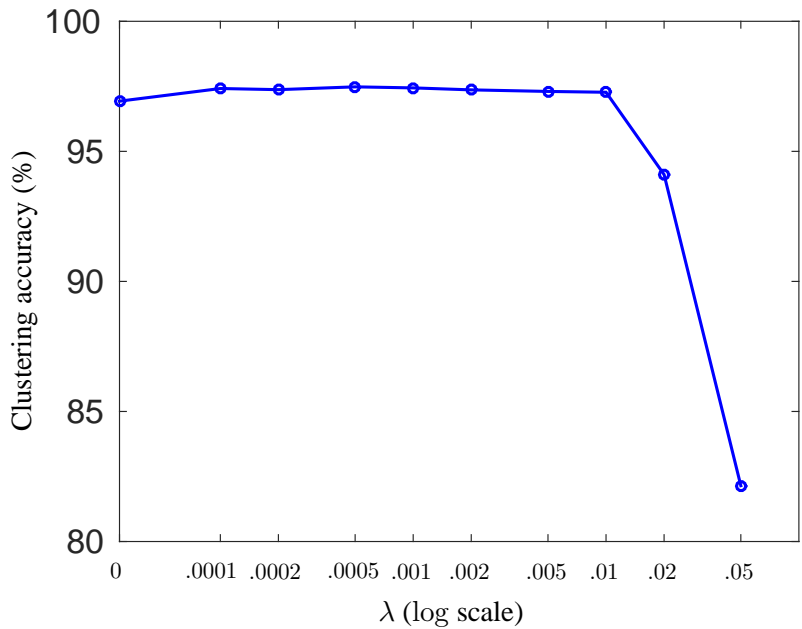}
\caption{Clustering accuracies (\%) of \DCCAE\ for different $\lambda$ values at $L=10$, with other hyperparameters set to their optimal values.  Each point gives the mean accuracy obtained with $5$ random seeds. $\lambda = 0.0005$ gives the best performance, but only by a very slight margin.
\jeff{any chance to reduce the font of the values of $\lambda$ to make the numbers not bump into each other?} \weiran{done}
}
\label{f:mnist_vary_lambda}
\end{figure}

For DNN-based models, we use SGD for optimization with minibatch size, learning rate and momentum tuned on the tuning set (more on this in Section~\ref{sec:expt_opt}). 
A small weight decay parameter of $10^{-4}$ is used for all layers. We monitor the objective on the tuning set for early stopping. For each algorithm, we select the model with the best AC on the tuning set, and report its results on the test set. The AC and NMI results (in percent) for each algorithm are given in Table~\ref{t:mnist_acc_nmi}. As a baseline, we also cluster the original $784$-dimensional view 1 images.


Next, we also measure the quality of the projections via classification experiments.  If the learned features are clustered well into classes, then one might expect that a simple linear classifier can achieve high accuracy on these projections. We train one-versus-one linear SVMs \citep{ChangLin11a} on the projected training set (now using the ground truth labels), and test on the projected test set, while using the projected tuning set for selecting the SVM hyperparameter (the penalty parameter for hinge loss). Test error rates on the optimal embedding of each algorithm (with highest AC) are provided in Table~\ref{t:mnist_acc_nmi} (last column). These error rates agree with the clustering results. 
Multi-view feature learning makes classification much easier on this task: Instead of using a heavily nonlinear classifier on the original inputs, a very simple linear classifier that can be trained efficiently on low-dimensional projections already achieves high accuracy.

All of the multi-view feature learning algorithms achieve some improvement over the baseline.  The nonlinear CCA algorithms all perform similarly, and significantly better than \SVAE, \CorrAE, and \DistAE. We also qualitatively investigate the features by embedding the projected features in 2D using $t$-SNE \citep{MaatenHinton08a}; the resulting visualizations are given in Figures~\ref{f:mnist_all_1} and~\ref{f:mnist_all_2}. Overall, the visual class separation qualitatively agrees with the relative clustering and classification performance in Table~\ref{t:mnist_acc_nmi}.

In the embedding of input images (Figure~\ref{f:mnist_all_1} (a)), samples of each digit form an approximately one-dimensional, stripe-shaped manifold, and the degree of freedom along each manifold corresponds roughly to the variation in rotation angle (see Figure~\ref{f:mnist_all_1} (a')).
This degree of freedom does not change the identity of the image, which is common to both views. Projections by \SVAE/\CorrAE/\DistAE\ do achieve somewhat better separation for some classes, but the unwanted rotation variation is still prominent in the embeddings. On the other hand, without using any label information and with only paired noisy images, the nonlinear CCA algorithms manage to map digits of the same identity to similar locations while suppressing the rotational variation and separating images of different identities. 
Linear CCA also approximates the same behavior, but fails to separate the classes, presumably because the input variations are too complex to be captured by only linear mappings. Overall, \DCCAE\ gives the cleanest embedding, with different digits pushed far apart and good separation achieved.

The different behavior of CCA-based methods from \SVAE/\CorrAE/\DistAE\ suggests two things. First, when the inputs are noisy, reconstructing the input faithfully may lead to unwanted degrees of freedom in the features (\DCCAE\ tends to select a relatively small trade-off parameter $\lambda=10^{-3}$ or $10^{-2}$), further supporting that it is not necessary to fully minimize reconstruction error. We show the clustering accuracy of \DCCAE\ at $L=10$ for different $\lambda$ values in Figure~\ref{f:mnist_vary_lambda}. 
Second, the hard CCA constraints, which enforce uncorrelatedness between different feature dimensions, appear essential to the success of CCA-based methods; these constraints are the difference between \DCCAE\ and \CorrAE.
However, the constraints without the multi-view objective seem to be insufficient.  To see this, we also visualize a $10$-dimensional locally linear embedding (LLE, \citealp{RoweisSaul00a}) of the test images in Fig.~\ref{f:mnist_all_1} (b). LLE satisfies the same uncorrelatedness constraints as in CCA-based methods, but without access to the second view, it does not separate the classes as nicely.

\subsection{Acoustic-articulatory data for speech recognition}
\label{sec:xrmb}

We next experiment with the Wisconsin X-Ray Microbeam (XRMB) corpus \citep{Westbur94a} of simultaneously recorded speech and articulatory measurements from 47 American English speakers.
Multi-view feature learning via CCA/KCCA/DCCA has previously been shown to improve phonetic recognition performance when tested on audio alone
\citep{AroraLivesc13a,Wang_15a}. 

We follow the setup of \citet{AroraLivesc13a} and use the learned features for speaker-independent phonetic recognition. 
Similarly to \citet{AroraLivesc13a}, the inputs to multi-view feature learning are acoustic features (39D features consisting of mel frequency cepstral coefficients (MFCCs) and their first and second derivatives) and articulatory features (horizontal/vertical displacement of 8 pellets attached to different parts of the vocal tract) concatenated over a 7-frame window around each frame, giving 273D acoustic inputs and 112D articulatory inputs for each view. 

We split the XRMB speakers into disjoint sets of 35/8/2/2 speakers for feature learning/recognizer training/tuning/testing.  The 35 speakers for feature learning are fixed; the remaining 12 are used in a 6-fold experiment (recognizer training on 
8 speakers, tuning on 2 speakers, and testing on the remaining 2 speakers).
Each speaker has roughly $50K$ frames, giving 1.43M multi-view training frames; this is a much larger training set than those used in previous work on this data set. We remove the per-speaker mean and variance of the articulatory measurements for each training speaker. All of the learned feature types are used in a tandem approach \citep{Herman_00c}, i.e., they are appended to the original 39D features and used in a standard hidden Markov model (HMM)-based recognizer with Gaussian mixture model observation distributions.  
The baseline is the recognition performance using the original MFCC features.
The recognizer has one 3-state left-to-right HMM per phone, using the same language model as in \citet{AroraLivesc13a}.

For each fold, we select the best hyperparameters based on recognition accuracy on the tuning speakers, and use the corresponding learned model for the test speakers. As before, models based on neural networks are trained via SGD with the optimization parameters tuned by grid search.  Here we do not use pre-training for weight initialization. A small weight decay parameter of $5\times 10^{-4}$ is used for all layers. For each algorithm, the feature dimensionality $L$ is tuned over $\{30, 50, 70\}$. For DNN-based models, we use hidden layers of $1500$ ReLUs. For \DCCA, we vary the network depths (up to 3 nonlinear hidden layers) of each view. In the best \DCCA\ architecture, $\f$ has $3$ ReLU layers of $1500$ units followed by a linear output layer while $\g$ has only a linear output layer.

For \CorrAE/\DistAE/\DCCAE, we use the same architecture of \DCCA\ for the encoders, and we set the decoders to have symmetric architectures to the encoders.  For this domain, we find that the best choice of architecture for the encoders/decoders for View 2 is linear while for View 1 it is typically three layers deep.  For \SVAE, the encoder $\f$ is similarly deep and the View 1 decoder $\p$ has the symmetric architecture, while its View 2 decoder $\q$ was set to linear to match the best choice for the other methods. 
We fix $(r_x,r_y)$ to small values as before. The trade-off parameter $\lambda$ is tuned for each algorithm by grid search. 

For \RKCCA, we find it important to use a large number of random features $M$ to get a competitive result, consistent with the findings of \citet{Huang_14a} when using random Fourier features for speech data. We tune kernel widths at $M=5,\!000$ with \RKCCA, and test \RKCCA\ with $M=30,\!000$ (the largest $M$ we could afford to obtain an exact SVD solution on a workstation with 32G main memory); 
We are not able to obtain results for \NKCCA\ with $M=30,\!000$ in 48 hours with our implementation, so we report its test performance at $M=20,\!000$ with the optimal \RKCCA\ hyperparameters. Notice that \RKCCA\ has about $14.6$
million parameters (random Gaussian samples + projection matrices from random Fourier features to the $L$-dimensional KCCA features,
which is more than the number of weight parameters in the largest DCCA model, so it is slower than DCCA for testing (the cost of computing test features is linear in the number of parameters for both KCCA and DNNs).

\begin{table}[t]
\centering
\caption{Mean and standard deviations of PERs over 6 folds obtained by each algorithm on the XRMB test speakers.}
\label{t:xrmb_pers}
\begin{tabular}{l|c}
\hline
Method & Mean (std) PER (\%)  \\
\hline \hline
Baseline & 34.8\ \ (4.5) \\
CCA & 26.7\ \ (5.0) \\
\SVAE\ & 29.0\ \ (4.7) \\
\CorrAE\ & 30.6\ \ (4.8) \\
\DistAE-1\ & 33.2\ \ (4.7) \\
\DistAE-2\ & 32.7\ \ (4.9) \\
\RKCCA & 26.0\ \ (4.4) \\
\NKCCA & 26.6\ \ (4.2) \\
\DCCA\ & \textbf{24.8\ \ (4.4)}\\
\DCCAE\ & \textbf{24.5\ \ (3.9)} \\ 
\hline
\end{tabular}
\end{table}

Phone error rates (PERs) obtained by different feature learning algorithms are given in Table~\ref{t:xrmb_pers}. We see the same pattern as on MNIST:  Nonlinear CCA-based algorithms outperform \SVAE/\CorrAE/\DistAE.  Since the recognizer now is a nonlinear mapping (HMM), the performance of the linear CCA features is highly competitive. Again, \DCCAE\ tends to select a relatively small $\lambda$, indicating that the canonical correlation term is more important. 

\subsection{Multilingual data for word embeddings}
\label{sec:NLP}

\begin{table}[t]
\centering
\caption{Spearman's correlation ($\rho$) for word/bigram similarities.}
\label{t:bigram}
\begin{tabular}{l|cccccc|cc}
\hline
Method &  WS-353 & WS-SIM & WS-REL & RG-65 & MC-30 & MTurk-287 & AN &  VN  \\
\hline \hline
Baseline & 67.0 & 73.0 & 63.1 & 73.4 & 78.5 & 52.0 & 45.0 &  39.1 \\
CCA &    68.4 & 71.9 & 64.7 & 76.6 & 83.4 & 58.7 &  46.6 & 43.2  \\
\SVAE &   63.5 & 68.0 & 58.0 & 75.9 & \textbf{85.9} & 62.9 & 46.6 & \textbf{44.6}  \\
\CorrAE & 63.5 & 69.0 & 57.5 & 73.7& 81.5 & 61.6 & 43.4 & 42.0 \\
\DistAE-1  & 64.5 & 70.7 & 58.2 & 75.7 & 84.2 & 61.6 & 45.3 & 39.4 \\
\DistAE-2  & 64.9 & 70.4 & 61.9 & 75.3 & 82.5& 62.5 & 42.4 & 40.0 \\
\RKCCA & 68.7 & 70.5 & 62.5 & 70.3 & 80.8 & 60.8 & 46.4 &  42.9 \\
\NKCCA & \textbf{69.9} & 74.9 & \textbf{65.7} & \textbf{80.7} & 85.8 & 57.5 & 48.3 & 40.9 \\
\DCCA &  69.7 & \textbf{75.4} & 63.4 & 79.4 & 84.3 & \textbf{65.0} & 48.5 &  42.5  \\
\DCCAE & 69.7 & 75.4 & 64.4 & 79.4 & 84.7 & 65.0 & \textbf{49.1} &  43.2  \\
\hline
\end{tabular}
\end{table}

In this task, we learn a vectorial representation of English words from pairs of English-German word embeddings for improved semantic similarity. We follow the setup of \citet{FaruquiDyer14a} and use as inputs $640$-dimensional monolingual word vectors trained via latent semantic analysis on the WMT 2011 monolingual news corpora 
and use the same $36K$ English-German word pairs for multi-view learning. The learned mappings are applied to the original English word embeddings ($180K$ words) and the projections are used for evaluation. 

We evaluate learned features on two groups of tasks. The first group consists of the four word similarity tasks from \citet{FaruquiDyer14a}: WS-353 and the two splits WS-SIM and WS-REL, RG-65, MC-30, and MTurk-287. The second group of tasks uses the adjective-noun (AN) and verb-object (VN) subsets from the bigram similarity dataset of \citet{MitchelLapata10a}, and tuning and test splits (of size 649/1972) for each subset (we exclude the noun-noun subset as we find that the NN human annotations often reflect ``topical'' rather than ``functional'' similarity). We simply add the projections of the two words in each bigram to obtain an $L$-dimensional representation of the bigram, as done in prior work~\citep{BlacoeLapata12a}.
We compute the cosine similarity between the two vectors of each bigram pair, order the pairs by similarity, and report the Spearman's correlation ($\rho$) between the model's ranking and human rankings.


We tune the feature dimensionality $L$ over $\{128,384\}$; other hyperparameters are tuned as in previous experiments. DNN-based models use ReLU hidden layers of width $1,\!280$. A small weight decay parameter of $10^{-4}$ is used for all layers. We use two ReLU hidden layers for  encoders ($\f$ and $\g$), and try both linear and nonlinear networks with two hidden layers for decoders ($\p$ and $\q$). 
\RKCCA/\NKCCA\ are tested with $M=20,\!000$ using kernel widths tuned at $M=4,\!000$. We fix $r_x=r_y=10^{-4}$ for nonlinear CCAs and tune them over $\{10^{-6},10^{-4},10^{-2},1,10^{2}\}$ for CCA.

For word similarity tasks, we simply report the highest  Spearman's correlation obtained by each algorithm.  For bigram similarity tasks,  we select for each algorithm the model with the highest Spearman's correlation on the 649 tuning bigram pairs, and we report its performance on the 1972 test pairs. The results are given in Table~\ref{t:bigram}. Unlike MNIST and XRMB, it is important for the features to reconstruct the input monolingual word embeddings well, as can be seen from the superior performance of \SVAE\ over \RKCCA/\NKCCA/\DCCA. This implies there is useful information in the original inputs that is not correlated across views. However, \DCCAE\ still performs the best on the AN task, in this case using a relatively large $\lambda=0.1$.
\jeff{Say why can't DCCAE ultimately do better if $\lambda$ is further increased?}
\weiran{My parameters are chosen with cross-validation. Unfortunately, the NLP datasets we used here are somewhat small. Sometimes there is problem of overfitting on validation set. I remember this is the case for some task.}


\subsection{Empirical analysis of \DCCA\ optimization}
\label{sec:expt_opt} 

\begin{figure}[t]
\centering
\psfrag{correlation}[c][c]{\large Total Canon. Corr.}
\psfrag{hours}{\large hours}
\psfrag{SGD 100}{\large STO 100}
\psfrag{SGD 200}{\large STO 200}
\psfrag{SGD 300}{\large STO 300}
\psfrag{SGD 400}{\large STO 400}
\psfrag{SGD 500}{\large STO 500}
\psfrag{SGD 750}{\large STO 750}
\psfrag{SGD 1000}{\large STO 1000}
\psfrag{NOI 100}{\large NOI 100}
\psfrag{L-BFGS}{\large L-BFGS}
\includegraphics[width=0.7\linewidth]{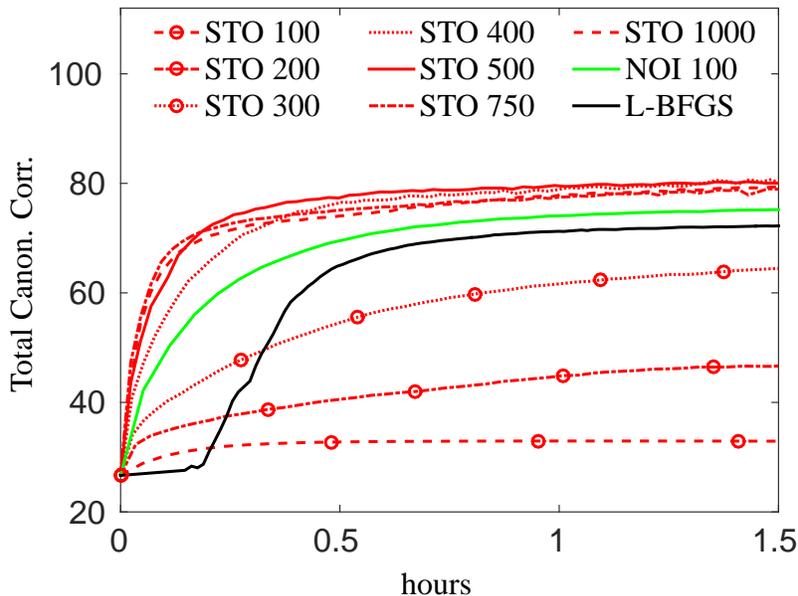}
\caption{Learning curves (total canonical correlation vs.~training time) of \DCCA\ on the XRMB `JW11' tuning set.
The maximum correlation is equal to the dimensionality (112).
Each marker corresponds to one epoch (one pass over the training data) for STO or NOI, 
or one iteration for L-BFGS.  ``STO $n$'' = stochastic optimization with minibatch size $n$.
}
\label{f:learningcurve}
\end{figure}

We now explore the issue of stochastic optimization for \DCCA\ as discussed in Section~\ref{sec:DCCA}. We use the same XRMB dataset as in the acoustic-articulatory experiment of Section~\ref{sec:xrmb}. In this experiment, we select utterances from a single speaker `JW11' and divide them into training/tuning/test splits of roughly $30K$/$11K$/$9K$ pairs of acoustic and articulatory frames.\footnote{Our split of the data is the same as the one used by \citet{Andrew_13a}. We note that \citet{Lopez_14b} used the same speaker in their experiments, but they randomly shuffled all $50K$ frames before creating the splits. We suspect that \DCCA\ (as well as their own algorithm) were under-tuned in their experiments. We ran experiments on a randomly shuffled dataset with careful tuning of kernel widths for \RKCCA/\NKCCA~(using rank $M=6000$) and obtained canonical correlations of 99.2/105.6/107.6 for \RKCCA/\NKCCA/\DCCA, which are better than those reported in \citet{Lopez_14b}.}
Since the sole purpose of this experiment is to study the optimization of nonlinear CCA algorithms rather than the usefulness of the features, we do not carry out any down-stream tasks or careful model selection for that purpose (\eg, a search for feature dimensionality $L$).

We now consider the effect of our stochastic optimization procedure for \DCCA, denoted STO below, and demonstrate the importance of minibatch size. We use a 3-layer architecture where the acoustic and articulatory networks have two hidden layers of 1800 and 1200 rectified linear units (ReLUs) respectively, and the output dimensionality (and therefore the maximum possible total canonical correlation over dimensions) is $L=112$. We use a small weight decay $\gamma=10^{-4}$, and do grid search for several hyperparameters: $r_x, r_y \in \{10^{-4},\ 10^{-2},\ 1,\ 10^2\}$, constant learning rate in $\{10^{-4},\ 10^{-3},\ 10^{-2},\ 10^{-1}\}$, fixed momentum in $\{0,\ 0.5,\ 0.9,\ 0.95,\ 0.99\}$, and minibatch size in $\{100,\ 200,\ 300,\ 400,\ 500,\ 750,\ 1000\}$. After learning the projection mappings on the training set, we apply them to the tuning/test set to obtain projections, and 
measure the canonical correlation between views. 
Figure~\ref{f:learningcurve} shows the learning curves on the tuning set for different minibatch sizes, each using the optimal values for the other hyperparameters. It is clear that for small minibatches ($100$, $200$), the objective quickly plateaus at a low value, whereas for large enough minibatch size, there is always a steep increase at the beginning, which is a known advantage of stochastic first-order algorithms \citep{BottouBousquet08a}, and a wide range of learning rate/momentum give very similar results. The reason for such behavior is that the stochastic estimate of the \DCCA\ objective becomes more accurate as minibatch size $n$ increases. We provide theoretical analysis of the error between the true objective and its stochastic estimate in Appendix~\ref{sec:append-sgd}. 

Recently, \citet{Wang_15c} have proposed a nonlinear orthogonal iterations (NOI) algorithm for the \DCCA\ objective which extends the alternating least squares procedure~\citep{GolubZha95a,LuFoster14a} and its stochastic version~\citep{Ma_15b} for CCA. Each iteration of the NOI algorithm adaptively estimates the covariance matrices of the projections of each view, whitens the projections of a minibatch using the estimated covariance matrices, and takes a gradient step over DNN weight parameters of the nonlinear least squares problems of regressing each view's input against the whitened projection of the other view for the minibatch. The advantage of NOI is that it performs well with smaller minibatch sizes and thus reduces memory consumption. \citet{Wang_15c} have shown that NOI can achieve the same objective value as STO using smaller minibatches. For the problems considered in this paper, however, each epoch of NOI takes a longer time than that of STO, due to the whitening operations at each iteration (involving eigenvalue decomposition of $L\times L$ covariance matrices). The smaller the minibatch size we use in NOI, the more frequently we run such operations. We report the learning curve of NOI with minibatch size $n=100$ while all other hyper-parameters are tuned similarly to STO.  NOI can eventually reach the same objective as STO given more training time.  Ultimately, the choice of optimization algorithm will depend on the task-specific data and time/memory constraints. 

Finally, we also train the same model using batch training with L-BFGS\footnote{We use the L-BFGS implementation of \citet{Schmid12a}, which includes a good line-search procedure.} with the same random initial weight parameters and tune $(r_x,r_y)$ on the same grid. While L-BFGS does well on the training set, its performance on tuning/test is usually worse than that of stochastic optimization with reasonable hyperparameters. \jeff{Why is this? Also, is this the fully batch case? Provide more details of this instance here, enough so that it is likely the experiment could be repeated by others.} \weiran{Stochastic algorithms tend to generalize better. I had code online that more or less replicate these experiments.} \kl{added ``batch training with L-BFGS'' to clarify}

\begin{figure}[t]
\centering
\psfrag{corr}[cb][c]{\large Total Canon. Corr.}
\psfrag{M}{\large $M$}
\psfrag{NKCCA}{\NKCCA}
\psfrag{RKCCA}{\RKCCA}
\includegraphics[width=0.7\linewidth]{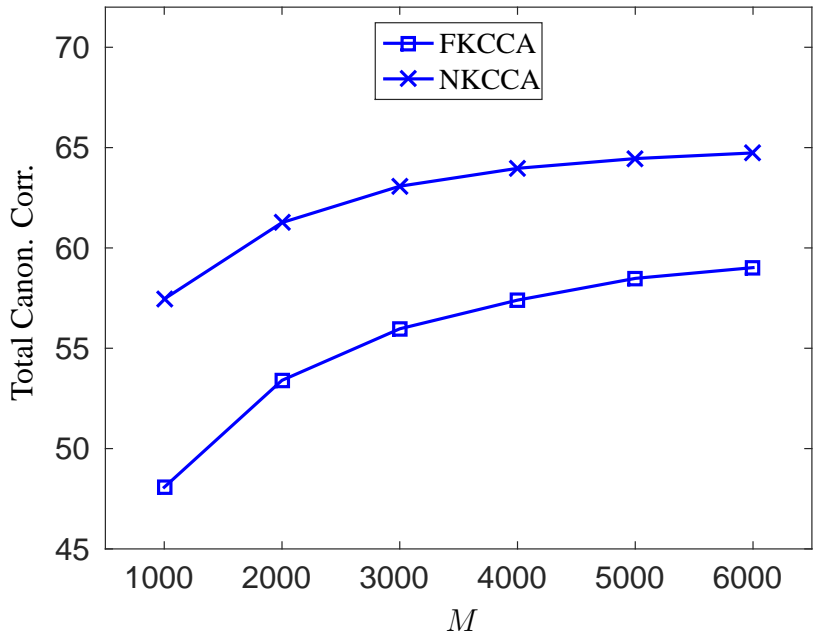}
\caption{Total canonical correlation achieved by approximate KCCAs on the XRMB `JW11' tuning set for different rank $M$.
}
\label{f:kcca_M}
\end{figure}

\begin{figure}[t]
\centering
\psfrag{corr}[cb][c]{\large Total Canon. Corr.}
\psfrag{time}[][]{\large hours}
\psfrag{NKCCA}{\NKCCA}
\psfrag{RKCCA}{\RKCCA}
\psfrag{DCCA LBFGS}{\DCCA\ L-BFGS}
\psfrag{DCCA STO}{\DCCA\ STO}
\includegraphics[width=0.7\linewidth]{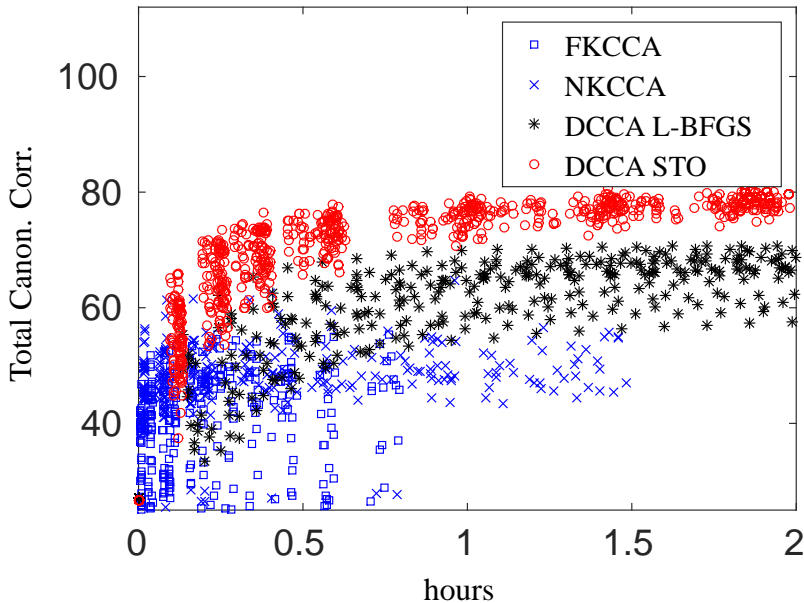}
\caption{Total canonical correlation vs.~running time (hours) achieved by approximate KCCAs (varying rank $M$ and kernel widths), by \DCCA\ with L-BFGS (varying $(r_x,r_y)$), and by \DCCA\ with stochastic optimization (varying $(r_x,r_y)$, minibatch size, learning rate, and momentum, run for up to $100$ epochs) 
on the XRMB `JW11' tuning set. To avoid clutter, we show \RKCCA/\NKCCA\ results for the best $300$ hyperparameter combinations, and \DCCA\ results for the best $100$ 
hyperparameter combinations. Since both L-BFGS and stochastic optimization are iterative algorithms, we also show the correlation obtained at different numbers of epochs. 
}
\label{f:scatterJW11}
\end{figure}

We also train \RKCCA\ and \NKCCA\ on this dataset. We tune $(r_x,r_y)$ on the same grid as for \DCCA, tune the kernel width for each view, and vary the approximation rank $M$ from $1000$ to $6000$ for each method. We plot the best total canonical correlation achieved by each algorithm on the tuning set as a function of $M$ in Figure~\ref{f:kcca_M}. Clearly both algorithms require relatively large $M$ to perform well on this task, but the return in further increasing $M$ is diminishing. \NKCCA\ achieves better results than \RKCCA, although forming the low-rank approximation also becomes more costly as $M$ increases. 
We plot the total canonical correlation achieved by each algorithm at different running times with various hyperparameters in Figure~\ref{f:scatterJW11}. 

\begin{table}[t]
\centering
\caption{Total canonical correlation achieved by each algorithm on the XRMB `JW11' test set.
The maximum achievable value here is 112.} 
\label{t:jw11_corr}
\begin{tabular}{|c||c|c|c|c|c|}\hline
 & CCA & \RKCCA\ & \NKCCA\ & \DCCA\ L-BFGS & \DCCA\ Stochastic \\ \hline
Canon. Corr. & 23.4 &  68.2 & 74.3 & 83.7 & 89.6 \\ \hline
\end{tabular}
\end{table}

Finally, we select for each algorithm the best model on the tuning set and give the corresponding total canonical correlation on the test set in Table~\ref{t:jw11_corr}.  


\section{Conclusion}
\label{sec:conclusion}

We have explored several approaches in the space of DNN-based multi-view representation learning.  We have found that on several tasks, CCA-based models outperform \autoencoder-based models (\SVAE) and models based on between-view squared distance (\DistAE) or correlation (\CorrAE) instead of canonical correlation.  The best overall performer is a new \DCCA\ extension, deep canonically correlated autoencoders (\DCCAE).  We have studied these objectives in the context of DNNs, but we expect that the same trends should apply also to other network architectures such as convolutional~\citep{Lecun_98a} and recurrent~\citep{elman1990finding,hochreiter1997long} networks, and this is one direction for future work.

In light of the empirical results, it is interesting to consider again the main features of each type of objective and corresponding constraints.  Autoencoder-based approaches are based on the idea that the learned features should be able to accurately reconstruct the inputs (in the case of multi-view learning, the inputs in both views).  The CCA objective, on the other hand, focuses on how well each view's representation predicts the other's, ignoring the ability to reconstruct each view.  CCA is expected to perform well for clustering and classification when the two views are uncorrelated given the class label~\citep{Chaudh_09a}. 
The noisy MNIST dataset used here simulates exactly this scenario, and indeed this is the task where deep CCA outperforms other objectives by the largest margins. 
\jeff{I don't understand this. CCA (the linear case) does not outperform the other
objectives. Please clarify this statement.}
\kl{changed ``CCA'' to ``deep CCA'' since we don't mean the linear case, but rather are referring to the CCA objective (on top of a deep network) vs. other objectives}
Even in the other tasks, however, there usually seems to be only a small advantage to being able to reconstruct the inputs faithfully.

The constraints in the various methods also have an important effect.  The performance difference between \DCCA\ and \CorrAE\ demonstrates that uncorrelatedness between learned dimensions is important.  On the other hand, the stronger \DCCA\ constraint may still not be sufficiently strong; an even better constraint may be to require the learned dimensions to be independent (or approximately so), and this is an interesting avenue for future work.

We believe the applicability of the \DCCA\ objective goes beyond the unsupervised feature learning setting. For example, it can be used as a data-dependent regularizer in supervised or semi-supervised learning settings where we have some labeled data as well as multi-view observations.  The usefulness of CCA in such settings has previously been analyzed theoretically~\citep{KakadeFoster07a,Mcwill_13a} and has begun to be explored experimentally~\citep{AroraLivesc14a}.

\appendix

\section{Analysis of stochastic optimization for \DCCA}
\label{sec:append-sgd}

Our SGD-like optimization for DCCA works as follows. We randomly pick a minibatch of $n$ training pairs $\{(\x_i,\y_i)\}_{i=1}^n$, 
feed them forward through the networks $\f$ and $\g$ to obtain outputs $\{(\f_i,\g_i)\}_{i=1}^n$ where $\f_i=\f(\x_i)\in \bbR^{d_x}$ and $\g_i=\g(\y_i)\in \bbR^{d_y}$, and estimate the covariances of network outputs based on these samples, namely $\hat{\bSigma}_{xx}=\frac{1}{n}\sum_{i=1}^n \f_i \f_i^\top + r_x \I$, $\hat{\bSigma}_{yy}=\frac{1}{n}\sum_{i=1}^n \g_i \g_i^\top + r_y \I$, and $\hat{\bSigma}_{xy}=\frac{1}{n}\sum_{i=1}^n \f_i \g_i^\top$, and finally compute a ``partial'' objective which is the sum of the top $L$ singular values of $\hat{\T} = \hat{\bSigma}_{xx}^{-1/2} \hat{\bSigma}_{xy} \hat{\bSigma}_{yy}^{-1/2}$, together with its gradient with respect to the network outputs and the weights at each layer of each network (see Section~3 of \citealp{Andrew_13a} for the gradient formulas). 

We denote by $\bTheta^{(n)}$ the ``partial'' objective based on the $n$ samples, i.e.
\begin{align*}
\bTheta^{(n)}=\sum_{k=1}^L \sigma_k (\hat{\T}) = \sum_{k=1}^L \sigma_k (\hat{\bSigma}_{xx}^{-1/2} \hat{\bSigma}_{xy} \hat{\bSigma}_{yy}^{-1/2}).
\end{align*}
And recall that our true objective function is 
\begin{align*}
\bTheta=\sum_{k=1}^L \sigma_k (\T) = \sum_{k=1}^L \sigma_k ({\bSigma}_{xx}^{-1/2}  {\bSigma}_{xy} {\bSigma}_{yy}^{-1/2}),
\end{align*}
where ${\bSigma}_{xx}=\frac{1}{n}\sum_{i=1}^N \f_i \f_i^\top + r_x \I$, ${\bSigma}_{yy}=\frac{1}{n}\sum_{i=1}^N \g_i \g_i^\top + r_y \I$, and ${\bSigma}_{xy}=\frac{1}{n}\sum_{i=1}^N \f_i \g_i^\top$ are computed over the entire training set.

An important observation is that 
\begin{align*}
\bbE\,\bTheta^{(n)} \neq \bTheta,
\end{align*}
where the expectation is taken over random selection of the $n$ training samples (all expectations in this section are taken over the random sampling process). This fact indicates that in expectation, the ``partial'' objective we compute is not equal to the true objective we want to optimize and thus the naive implementation is not really a stochastic gradient algorithm (which requires the gradient estimate to be unbiased). The reason why we can not in general have $\bbE\,\bTheta^{(n)}=\bTheta$ for the naive implementation is that there exist three nonlinear operations in our objective:
the sum of singular values operations, the multiplication of three matrices, and the inverse square root operations for auto-covariance matrices. These operations keep us from moving the expectation inside, even though we do have $\bbE\,\hat{\bSigma}_{xx}={\bSigma}_{xx}$, $\bbE\,\hat{\bSigma}_{yy}={\bSigma}_{yy}$, and $\bbE\,\hat{\bSigma}_{xy} ={\bSigma}_{xy}$.

In the following, we try to bound the error between $\bbE\,\bTheta^{(n)}$ and $\bTheta$ for a slightly modified version of the naive implementation. Specifically, we make the two assumptions below which make the analysis easier.
\begin{itemize}
\item \textbf{A1}: The samples used for estimating ${\bSigma}_{xx}$, ${\bSigma}_{yy}$ and ${\bSigma}_{xy}$ are chosen independently from each other, each by sampling examples (for ${\bSigma}_{xx}$ and ${\bSigma}_{yy}$) or example pairs (for ${\bSigma}_{xy}$) uniformly at random with replacement from the training set. In practice, we could randomly pick $n$ samples of $\f(\x)$ for computing  $\hat{\bSigma}_{xx}$, another $n$ samples of $\g(\y)$ for computing $\hat{\bSigma}_{yy}$, and $n$ pairs of samples of $(\f(\x), \g(\y))$ for computing $\hat{\bSigma}_{xy}$, and the computational cost of this modified procedure is twice the cost of the original naive implementation. This simple modification allows the expectation of matrix multiplications to factorize. 
\item \textbf{A2}: Additionally, we assume an upper bound on the 
magnitude of the neural network outputs, namely, $\max( \norm{\f_i}^2, \norm{\g_i}^2 ) \le B$, $\forall i$. This holds when nonlinear activations with a bounded range are used; \eg, with logistic sigmoid or hyperbolic tangent activations, the upper bound $B$ can be set to the output dimensionality (because the squared activation is bounded by $1$ for each output unit), or when the inputs themselves are bounded and the functions $\f(\x)$ and $\g(\y)$ are Lipschitz.
\jeff{Doesn't it also hold if the inputs vectors are themselves bounded, even for the RelU non-linearity, as long as the weights are also bounded?} \kl{yes, wording added}

\end{itemize}

Our analysis is based on the following formulation of the linear CCA solution~\citep{Borga01a}. Notice that the optimal $(\U, \V)$ for the CCA objective \eqref{e:dcca} satisfies
\begin{align} \label{e:cca-svd}
\left[\begin{array}{cc} \0 & {\T} \\ {\T}^\top & \0 \end{array}\right] \left[\begin{array}{c} {\bSigma}_{xx}^{1/2} \U \\ {\bSigma}_{yy}^{1/2}  \V \end{array}\right] = 
\left[\begin{array}{c}{\bSigma}_{xx}^{1/2} \U \\ {\bSigma}_{yy}^{1/2}  \V\end{array}\right]
{\bSigma},
\end{align}
where ${\bSigma}$ contains the top $L$ singular values of ${\T}$ on the diagonal. This is easy to verify as ${\bSigma}_{xx}^{1/2} \U$ and ${\bSigma}_{yy}^{1/2} \V$ contain the top left/right singular vectors of ${\T}$ (see, \eg, Section~2 of \citealp{Andrew_13a}).

Multiplying both sides of \eqref{e:cca-svd} by $\left[\begin{array}{cc} {\bSigma}_{xx}^{-1/2} & \0 \\ \0 & {\bSigma}_{yy}^{-1/2} \end{array}\right]$ gives
\begin{align*}
\left[\begin{array}{cc} {\bSigma}_{xx}^{-1} & \0 \\ \0 &  {\bSigma}_{yy}^{-1} \end{array}\right] \left[\begin{array}{cc} \0 & {\bSigma}_{xy} \\ {\bSigma}_{xy}^\top & \0 \end{array}\right] \left[\begin{array}{c} \U \\ \V \end{array}\right] =\left[\begin{array}{c} \U \\ \V \end{array}\right] {\bSigma},
\end{align*}
which implies $\left[\begin{array}{c} \U \\ \V \end{array}\right]$ correspond to the top $L$ eigenvectors of the matrix ${\A}^{-1} {\B}$, where
\begin{align*}
{\A}=\left[\begin{array}{cc} {\bSigma}_{xx} & \0 \\ \0 &  {\bSigma}_{yy} \end{array}\right],  \qquad
{\B}=\left[\begin{array}{cc} \0 & {\bSigma}_{xy} \\ {\bSigma}_{xy}^\top & \0 \end{array}\right].
\end{align*}
(It is straightforward to check that eigenvalues of ${\A}^{-1} {\B}$ are $\pm \sigma_1 ({\T}), \pm \sigma_2 ({\T}), \dots.$) 

When using minibatches, the estimate of the objective is of course computed based on the randomly chosen subset of samples, and is the sum of the top eigenvalues of the matrix $\hat{\A}^{-1} \hat{\B}$, where 
\begin{align*}
\hat{\A}=\left[\begin{array}{cc} \hat{\bSigma}_{xx} & \0 \\ \0 &  \hat{\bSigma}_{yy} \end{array}\right],  \qquad
\hat{\B}=\left[\begin{array}{cc} \0 & \hat{\bSigma}_{xy} \\ \hat{\bSigma}_{xy}^\top & \0 \end{array}\right].
\end{align*}

We now try to bound the expected error between $\hat{\A}^{-1} \hat{\B}$ and ${\A}^{-1} {\B}$ measured in spectral norm. \weiran{The following is newly added.} An important tool in our analysis is the matrix Bernstein inequality below.
\begin{lem}[Matrix Bernstein, Theorem~1.6.2 of \citealp{Tropp12a}]\label{lem:bernstein} Let $\A_1,\dots,\A_n$ be independent random matrices with common dimension $d_1\times d_2$. Assume that each matrix has bounded deviation from its mean:
\begin{gather*}
\norm{ \A_k - \bbE\,\A_k } \le R \qquad \forall k=1,\dots,n.
\end{gather*}
Form the sum $\B=\sum_{k=1}^n \A_k$, and introduce a variance parameter
\begin{gather*}
\sigma^2 = \max \left\{ \norm{\bbE\,\left[(\B-\bbE\,\B) (\B-\bbE\,\B)^\top\right]}, \norm{\bbE\,\left[(\B-\bbE\,\B)^\top (\B-\bbE\,\B)\right]} \right\}.
\end{gather*}
Then 
\begin{gather*}
\bbE\,\norm{ \B-\bbE\,\B } \le \sqrt{2 \sigma^2 \log (d_1 + d_2)} + \frac{1}{3} R \log (d_1 + d_2).
\end{gather*}
\end{lem}

The following result and its proof are similar to that of \citet[Theorem.~4]{Lopez_14b}. \weiran{Although I believe their bound misses a term regarding $\norm{\bbE \E_i}$. The problem was that for $n\le N$, $\bbE\,\left[\hat{\bSigma}_{xx}^{-1}\right]  - {\bSigma}_{xx}^{-1} \neq \0$. }

\begin{thm} Assume that \textbf{\emph{A1}} and \textbf{\emph{A2}} hold for our stochastic estimate of the true DCCA objective, and assume that the eigenvalues of autocovariance matrices $\hat{\bSigma}_{xx}$ and $\hat{\bSigma}_{yy}$ are lower bounded by $\gamma_x>0$ and $\gamma_y>0$ respectively, i.e.,
\begin{align*}
\hat{\bSigma}_{xx} \succeq \gamma_x \I, \quad\qquad \hat{\bSigma}_{yy} \succeq \gamma_y \I.
\end{align*}
Then we have the following bound for the expected error:
\begin{gather} \textstyle
\bbE\,\norm{\hat{\A}^{-1} \hat{\B} - {\A}^{-1} {\B}} \le \max \left\{ e_1, e_2 \right\}  
\end{gather}
where the expectation is taken over random selection of training samples as described in \textbf{\emph{A1}}, and 
\begin{align*} \textstyle
e_1 & = \frac{B}{\gamma_x^2} \left( \sqrt{ \frac{2B^2 \log(2 d_x)}{n}} + \frac{2B\log(2 d_x)}{3n} \right)  + \frac{1}{\gamma_x} \left( \sqrt{ \frac{2B^2 \log(d_x+d_y)}{n}} + \frac{2B\log(d_x+d_y)}{3n} \right), \\
e_2 & = \frac{B}{\gamma_y^2} \left( \sqrt{ \frac{2B^2 \log(2 d_y)}{n}} + \frac{2B\log(2 d_y)}{3n} \right)  + \frac{1}{\gamma_y} \left( \sqrt{ \frac{2B^2 \log(d_x+d_y)}{n}} + \frac{2B\log(d_x+d_y)}{3n} \right).
\end{align*}
\end{thm}

\begin{proof}
Since $\hat{\A}^{-1} \hat{\B} - {\A}^{-1} {\B}$ is block diagonal, its spectral norm is bounded by the maximum of the spectral norms of the individual blocks: 
\begin{align*}
\bbE\, \norm{\hat{\A}^{-1} \hat{\B} - {\A}^{-1} {\B}}   \le
\max\left\{
\bbE\, \norm{\hat{\bSigma}_{xx}^{-1} \hat{\bSigma}_{xy} - {\bSigma}_{xx}^{-1} {\bSigma}_{xy}},\ \bbE\, \norm{\hat{\bSigma}_{yy}^{-1} \hat{\bSigma}_{xy}^\top - {\bSigma}_{yy}^{-1} {\bSigma}_{xy}^\top} 
\right\}.
\end{align*}
We now focus on analyzing the first term in the above maximum as the other term follows analogously. 
Define the individual error terms
\begin{align*}
\E_i=\frac{1}{n} (\hat{\bSigma}_{xx}^{-1} \f_i \g_i^\top - {\bSigma}_{xx}^{-1} {\bSigma}_{xy}), \qquad
\E=\sum_{i=1}^n \E_i.
\end{align*}
Thus our goal is to bound $\norm{\hat{\bSigma}_{xx}^{-1} \hat{\bSigma}_{xy}- {\bSigma}_{xx}^{-1} {\bSigma}_{xy}}=\norm{\E}$. 

Notice that the matrices ${\bSigma}_{xx}$ and ${\bSigma}_{xy}$ are not random. And due to assumption \textbf{A1}, the samples used for estimating $\hat{\bSigma}_{xx}$ and $\hat{\bSigma}_{xy}$ are selected independently, so the expectation of $\hat{\bSigma}_{xx}^{-1} \f_i \g_i^\top$ factorizes. Therefore we have 
\begin{align*}
\bbE\,\E_i=\frac{1}{n} \left( \bbE\,\left[\hat{\bSigma}_{xx}^{-1}\right] {\bSigma}_{xy} - {\bSigma}_{xx}^{-1} {\bSigma}_{xy} \right),
\end{align*}
and the deviation of the individual error matrices from their expectation is
\begin{align*}
\Z_i:=\E_i - \bbE\,\E_i = \frac{1}{n} \left( \hat{\bSigma}_{xx}^{-1} \f_i \g_i^\top - \bbE\,\left[\hat{\bSigma}_{xx}^{-1}\right] {\bSigma}_{xy} \right).
\end{align*}
We would like to bound $\norm{\sum_{i=1}^n \Z_i}$ and $\norm{\bbE\,\E_i}$ separately using the matrix Bernstein inequality in Lemma~\ref{lem:bernstein}, and then use them to bound $e_1$. 
\jeff{List this inequality here for completeness.} \weiran{Added the lemma before theorem.}

\paragraph{Bounding $\norm{\sum_{i=1}^n \Z_i}$.} Notice that $\bbE\,\left[\Z_i\right]=\0$, and we can bound its norm (denoted $R$)
\begin{align*}
R:=\norm{\Z_i} & \le \frac{1}{n} \left(  \norm{\hat{\bSigma}_{xx}^{-1} \f_i \g_i^\top} + \norm{\bbE\,\left[\hat{\bSigma}_{xx}^{-1}\right] {\bSigma}_{xy}} \right) \\
& \le \frac{1}{n} \left( \norm{\hat{\bSigma}_{xx}^{-1}} \norm{\f_i \g_i^\top}+ \bbE\,\left[\norm{\hat{\bSigma}_{xx}^{-1}}\right] \norm{{\bSigma}_{xy}} \right) \\
& \le \frac{1}{n} \left(\frac{1}{\gamma_x} \norm{\f_i \g_i^\top} + \frac{1}{\gamma_x} \max_i \norm{\f_i \g_i^\top} \right) \\
& \le \frac{2}{n \gamma_x}  \max_i \norm{\f_i \g_i^\top} \\
& \le \frac{2}{n \gamma_x} \max_i \norm{\f_i}\norm{\g_i} \\
& \le \frac{2B}{n \gamma_x},
\end{align*}
where 
we have used the triangle inequality in the first inequality, and Jensen's inequality for the second 
and third inequality (since norms are convex functions). To apply the Matrix Bernstein inequality, we still need to bound the variance which is defined as
\begin{align*}
\sigma^2 :=&
\max\left\{ 
\norm{ \bbE\,\left[  \left(\sum_{i=1}^n \Z_i\right) \cdot \left(\sum_{i=1}^n \Z_i\right)^\top \right] },
\norm{  \bbE\,\left[ \left(\sum_{i=1}^n \Z_i\right)^\top \cdot \left(\sum_{i=1}^n \Z_i\right) \right] }
\right\} \\
=& \max\left\{ 
\norm{ \sum_{i=1}^n \bbE\,\left[ \Z_i \Z_i^\top \right]},
\norm{ \sum_{i=1}^n \bbE\,\left[ \Z_i^\top \Z_i \right]}
\right\},
\end{align*}
where we have used the fact that the $\Z_i$'s are independent and have mean zero in the second equality.

Let us consider an individual term in the summand of the second term:
\begin{align*}
 \Z_i^\top \Z_i  = & \frac{1}{n^2} \left( \g_i \f_i^\top \hat{\bSigma}_{xx}^{-2} \f_i \g_i^\top - \g_i \f_i^\top \hat{\bSigma}_{xx}^{-1} \bbE\,\left[ \hat{\bSigma}_{xx}^{-1} \right] \bSigma_{xy} \right. \\
 & \hspace{2em} \left. - \bSigma_{xy}^\top \bbE\,\left[ \hat{\bSigma}_{xx}^{-1} \right] \hat{\bSigma}_{xx}^{-1} \f_i \g_i^\top + \bSigma_{xy}^\top \bbE\,\left[ \hat{\bSigma}_{xx}^{-1} \right] \bbE\,\left[ \hat{\bSigma}_{xx}^{-1} \right] \bSigma_{xy} \right).
\end{align*}
Taking expectations 
we see that 
\begin{align*}
\bbE\,\left[ \Z_i^\top \Z_i  \right]=\frac{1}{n^2} \left( 
\bbE\, \left[ \g_i \f_i^\top \hat{\bSigma}_{xx}^{-2} \f_i \g_i^\top \right] -  \bSigma_{xy}^\top \left(\bbE\,\left[ \hat{\bSigma}_{xx}^{-1} \right]\right)^2 \bSigma_{xy} \right)
\end{align*}
where we have again used the assumption \textbf{A1}. 
Notice that $\bSigma_{xy}^\top \left( \bbE\,\left[ \hat{\bSigma}_{xx}^{-1} \right]\right)^2 \bSigma_{xy} $ is positive semi-definite, so subtracting it only decreases the spectral norm. We now take norms on both sides of the above equality and obtain
\begin{align*}
\norm{\bbE\,\left[ \Z_i^\top \Z_i  \right]} & \le \frac{1}{n^2} \norm{\bbE\, \left[ \g_i \f_i^\top \hat{\bSigma}_{xx}^{-2} \f_i \g_i^\top \right]} 
\le \frac{1}{n^2} \bbE\, \left[ \norm{ \g_i \f_i^\top \hat{\bSigma}_{xx}^{-2} \f_i \g_i^\top  \right]} \\
& \le \frac{1}{n^2} \bbE\, \left[ \norm{ \g_i \f_i^\top } \norm{  \hat{\bSigma}_{xx}^{-2} } \norm{ \f_i \g_i^\top }  \right] \\
& \le \frac{B^2}{n^2 \gamma_x^2},
\end{align*}
where Jensen's inequality is used in the second inequality.
A similar argument shows that
\begin{align*}
\norm{\bbE\,\left[ \Z_i \Z_i^\top \right]} \le \frac{1}{n^2} \norm{\bbE\, \left[ \hat{\bSigma}_{xx}^{-1} (  \f_i \g_i^\top \g_i \f_i^\top) \hat{\bSigma}_{xx}^{-1} \right]} \le \frac{B^2}{n^2 \gamma_x^2}.
\end{align*}
An invocation of the triangle inequality on the definition of $\sigma^2$ along with the above two bounds gives 
\begin{align*}
\sigma^2 \le \frac{B^2}{n \gamma_x^2}.
\end{align*}
We may now appeal to the matrix Bernstein inequality on the $d_x \times d_y$ matrices $\{ \Z_i \}_{i=1}^n$ to obtain the bound
\begin{align*}
\bbE\,\norm{ \sum_{i=1}^n \Z_i} & \le \sqrt{2\sigma^2 \log(d_x+d_y)} + \frac{1}{3} R \log(d_x+d_y) \\
&  = \frac{1}{\gamma_x} \left( \sqrt{ \frac{2B^2 \log(d_x+d_y)}{n}} + \frac{2B\log(d_x+d_y)}{3n} \right).
\end{align*}

\paragraph{Bounding $\bbE\,\E_i$.} Notice that 
\begin{align*}
\hat{\bSigma}_{xx}^{-1} - {\bSigma}_{xx}^{-1} = {\bSigma}_{xx}^{-1} ({\bSigma}_{xx}-\hat{\bSigma}_{xx}) \hat{\bSigma}_{xx}^{-1},
\end{align*}
so that
\begin{align*}
\norm{ \bbE\,\E_i} & = \frac{1}{n} \norm {\bbE\,\left[\hat{\bSigma}_{xx}^{-1}\right] {\bSigma}_{xy} - {\bSigma}_{xx}^{-1} {\bSigma}_{xy} } \\
& \le \frac{1}{n} \norm {\bbE\,\left[\hat{\bSigma}_{xx}^{-1} - {\bSigma}_{xx}^{-1} \right]} \norm{ {\bSigma}_{xy} }\\
& \le \frac{1}{n} \bbE\,\left[ \norm {\hat{\bSigma}_{xx}^{-1} - {\bSigma}_{xx}^{-1}} \right] \norm{ {\bSigma}_{xy} }\\
& \le \frac{B}{n \gamma_x^2} \bbE\,\norm{\hat{\bSigma}_{xx} - {\bSigma}_{xx}}.
\end{align*}
Using the same matrix Bernstein technique (now applied to $\sum_{i=1}^n \Z_i^\prime$ with $\Z_i^\prime=\frac{1}{n} \left(\f_i \f_i^\top -{\bSigma}_{xx} \right)$), we have 
\begin{align*}
\bbE\,\norm{\hat{\bSigma}_{xx} - {\bSigma}_{xx}} \le 
\sqrt{ \frac{2B^2 \log(2 d_x)}{n}} + \frac{2B\log(2 d_x)}{3n}.
\end{align*}
We are now ready to bound the quantity of interest as
\begin{align*}
e_1 & =\bbE\,\norm{ \hat{\bSigma}_{xx}^{-1} \hat{\bSigma}_{xy}  - {\bSigma}_{xx}^{-1} {\bSigma}_{xy} }  = \bbE\,\norm{\sum_{i=1}^n \E_i} \\
& = \bbE\,\norm{\sum_{i=1}^n \bbE\,\E_i + \sum_{i=1}^n \Z_i } \\
& \le \norm{ \sum_{i=1}^n \bbE\,\E_i } + \bbE\,\norm{ \sum_{i=1}^n \Z_i} \\
& \le \frac{B}{\gamma_x^2} \left( \sqrt{ \frac{2B^2 \log(2 d_x)}{n}} + \frac{2B\log(2 d_x)}{3n} \right) \\
& \qquad + \frac{1}{\gamma_x} \left( \sqrt{ \frac{2B^2 \log(d_x+d_y)}{n}} + \frac{2B\log(d_x+d_y)}{3n} \right). \\
 \end{align*}
The bound for $e_2=\bbE\,\norm{\hat{\bSigma}_{yy}^{-1} \hat{\bSigma}_{xy}^\top - {\bSigma}_{yy}^{-1} {\bSigma}_{xy}^\top}$ is completely 
analogous.
\end{proof}

Assuming that $d_x=d_y=d$ in our algorithm and $\gamma=\min(\gamma_x, \gamma_y)$, then we have the dependency of the error in spectral norm as
\begin{align*}
\bbE\, \norm{\hat{\A}^{-1} \hat{\B}  - {\A}^{-1} {\B}} 
\le 
\left(\frac{B^2}{\gamma^2}+\frac{B}{\gamma}\right) \left( \sqrt{\frac{2 \log(2d)}{n}}+ \frac{2 \log(2d)}{3 n} \right).
\end{align*}
As expected, the error decreases as we use large minibatch size $n$.  
Also, we observe the error decreases as $(\gamma_x, \gamma_y)$ increase. Notice that from the definition of $\hat{\bSigma}_{xx}$ and $\hat{\bSigma}_{yy}$ we are guaranteed that $\gamma_x\ge r_x$ and  $\gamma_y\ge r_y$. This means we can increase the regularization constants to improve the estimation error, and this is to be expected as larger $(r_x, r_y)$ render the auto-covariance matrices less relevant for estimating $\T$.
In practice, as long as one uses a minibatch size $n>d$, the auto-covariance matrices are typically non-singular and $(r_x, r_y)$ are underestimate of $(\gamma_x,\gamma_y)$.

\jeff{It would be useful to empirically compute these results on the data to estimate the degree to which the inequality in Thm.\ 1 is tight and hence could potentially be improved.} \kl{agreed.  But let's do that another time :) }

\section*{Acknowledgment}
The authors would like to thank Louis Goldstein for providing phonetic alignments for the data used in the recognition experiment; Manaal Faruqui, Chris Dyer, Ang Lu, Mohit Bansal, and Kevin Gimpel for sharing resources for the multi-lingual embedding experiments; Nati Srebro for input on the stochastic optimization of DCCA; and Geoff Hinton for helpful discussion on the CCA objective. 

\bibliography{jmlr14a}

\end{document}